\documentclass{EngCLS}

\title{Online Policy Evaluation for MDPs with Dynamic UBSR Measures}
\author{Weikai Wang\thanks{GERAD \& Department of Decision Sciences, HEC Montréal, Canada (\texttt{weikai.wang@hec.ca})} \quad 
Erick Delage\thanks{GERAD \& Department of Decision Sciences, HEC Montréal, Canada (\texttt{erick.delage@hec.ca})}
}
\date{\today}

\begin{document}

\maketitle

\begin{abstract}
    Developing efficient function-approximation methods for policy evaluation is a fundamental challenge in risk-aware reinforcement learning. Existing approaches either focus on restrictive classes of risk measures or rely on access to a simulator, limiting their applicability in fully online settings.
    In this work, we propose computationally efficient online learning algorithms for policy evaluation in Markov decision processes (MDPs) with dynamic utility-based shortfall risk (UBSR) measures under linear function approximation. Specifically, we introduce the UBSR-TD algorithm, establish conditions under which it converges almost surely, and develop several variants designed to accelerate convergence. Our formulation shows that existing policy evaluation algorithms for risk-neutral MDPs can be readily adapted to dynamic UBSR settings by incorporating a loss function into the temporal-difference error. Numerical experiments support our theoretical findings, and an application to a perishable inventory management problem with shelf-life uncertainty demonstrates the practical effectiveness of the proposed methods.
\end{abstract}

\section{Introduction}
From AlphaGo defeating the world champion of Go \citep{Silver2017nature} to the fine tuning of large language models \citep{Ouyang2022neurips}, reinforcement learning (RL) has achieved remarkable successes. A central component of RL is policy evaluation, which corresponds to assessing a given policy on a Markov decision process (MDP) and whose accuracy directly influences the effectiveness of subsequent policy improvement. In theory, policy evaluation can be achieved using classical value iteration methods from dynamic programming (DP). In practice, however, the environment is rarely known in full, data collection may be costly or limited, and the curse of dimensionality often renders exact DP methods computationally infeasible. These challenges have motivated efficient function approximation methods for estimating value functions from streaming or limited data, including temporal-difference (TD)-type algorithms \citep{Sutton2018book}.

The classical policy evaluation methods optimize the expected cumulative discounted cost and therefore implicitly assume that the decision-maker is risk-neutral. In real-world environments, however, policies may lead to severe adverse outcomes, particularly in high-stakes fields such as finance, autonomous driving, and robotics \citep{WangYH2022ai}, etc. The need to limit exposure to such outcomes has motivated the development of risk-aware RL methods, which incorporate risk measures into their objectives, statically or dynamically. Following the seminal framework of \cite{Artzner1999mf}, much of the literature has focused on the entropic risk measure (ERM) and coherent risk measures, largely because of their analytical tractability \citep{Tamar2015neurips, FeiYJ2021icml, Lam2022iclr, ZhangRY2024iclr}.

Despite these progresses, more general classes of risk measures remain relatively underexplored, limiting the capability of existing approaches in settings where the decision-maker seeks to capture diverse risk preferences through customized loss functions. In this context, utility-based shortfall risk (UBSR) has attracted increasing interest because of its flexibility. Defined in terms of a general loss function and a prescribed acceptability threshold, UBSR quantifies the minimum amount of additional capital required to reduce risk to an acceptable level \citep{Follmer2016book}. This general formulation includes both the expectation and the ERM as special cases. Empirical evidence also suggests that UBSR can effectively capture the heterogeneous risk preferences observed in human behavior \citep{Shen2014neuralcomp}. Moreover, UBSR is well suited to modeling extreme losses and is the only convex risk measure that is elicitable, a property of particular importance for backtesting \citep{Bellini2015qf}.

While possessing several attractive properties, UBSR has received relatively limited attention in RL, partly because its general lack of coherence gives rise to substantial analytical difficulties. 
Several studies have considered function approximation methods for MDPs with general dynamic risk measures \citep{Tamar2015neurips, YuPQ2018tac, Kose2021jmlr, Coache2024mf}. However, these approaches typically require a sampling oracle to evaluate the risk measure and therefore are not applicable in an online setting. 
Other methods estimate risk using online mini-batches or episodic trajectories \citep{HuangA2021arXiv, Marzban2023qf, Coache2023sifin, LuoYD2026arXiv}, but their convergence is not guaranteed. 
For dynamic UBSR specifically, \citet{Shen2014neuralcomp} develops an online Q-learning algorithm, but its function approximation version remains missing. 
To the best of our knowledge, two questions remain open for RL with dynamic risk measures under function approximation: how to perform fully online policy evaluation and whether the resulting algorithm can be shown to converge almost surely. \citet{Basu2008mor} resolves both for the average-ERM criterion, but under a formulation outside the discounted MDP setting considered here.

Motivated by this gap, this work aims to develop fully online policy evaluation algorithms for MDPs with dynamic UBSR under linear function approximation and establish their almost-sure convergence.
Our contributions can be described as follows:
\begin{itemize}
    \item 
    We develop UBSR-TD, the first fully online policy evaluation algorithm for MDPs with dynamic UBSR under linear function approximation. The method can be viewed as a natural extension of classical TD(0) to general loss functions. We show that UBSR-TD solves the fixed point of a projected risk-aware Bellman equation and establish its almost-sure convergence under suitable conditions. As far as we are aware, this is the first almost-sure convergence result for fully online policy evaluation in discounted MDPs with dynamic risk measures.

    \item TD-type algorithms can converge slowly without careful hyperparameter tuning \citep{Sutton2018book}. To mitigate this issue, we introduce two extensions of UBSR-TD that leverage historical and gradient information. We evaluate the proposed methods numerically across different risks and parameter settings. The results show that these variants can improve convergence over a broad range of scenarios. Our formulations indicate that existing policy evaluation algorithms can be easily adapted to the dynamic UBSR framework by incorporating a loss function into the TD error.

    \item We further apply our method to perishable inventory management with both demand and product shelf-life uncertainty, a growing field that has only been examined under risk neutrality \citep{Nahmias2011book, BuJZ2023mnsc, Abouee2026ijoc}. Existing risk-aware models employ conditional Value-at-Risk or ERM as risk measures \citep{FengYY2008opre, Yalcindag2020hcms, Pathy2024coa}, but do not account for shelf-life uncertainty. We formulate an approximate dynamic programming model under dynamic UBSR and employ our proposed policy evaluation algorithms within a policy iteration scheme. 
    Numerical experiments show that our approach consistently outperforms static, myopic, and risk-neutral benchmarks across diverse scenarios, highlighting the value of risk-aware decision-making.
\end{itemize}

The remainder of the paper is organized as follows. Section \ref{sec:Prelimiaries} introduces the basic definitions and formulates the policy evaluation problem for risk-aware MDPs under dynamic UBSR. Section \ref{sec:UBSRTD} presents the UBSR-TD algorithm, while Section \ref{sec:ConvergenceUBSRTD} provides its convergence analysis. Section \ref{sec:Extensions} introduces extensions of UBSR-TD designed to accelerate convergence. Section \ref{sec:Experiments} reports numerical experiments on algorithmic convergence and presents an application to a perishable inventory management problem. Finally, Section \ref{sec:Conclusion} concludes the paper and discusses future research directions.

\section{Preliminaries}\label{sec:Prelimiaries}
\textit{Notations:} Let $(\Omega,\mcF,\mbP)$ be a finite probability space, where $\Omega$ is the finite sample space, $\mcF := 2^{|\Omega|}$ is a $\sigma$-algebra on $\Omega$, and $\mbP$ is a probability measure on $(\Omega,\mcF)$. 
Let $\mcL(\Omega)$ denote the space of all real-valued measurable functions on $(\Omega,\mcF,\mbP)$. For $X, Y \in \mcL(\Omega)$, we write $X \ge Y$ if $X(\omega) \ge Y(\omega)$, for all $\omega\in \Omega$ and we say $X \ge Y$ almost surely (a.s.) if $X(\omega) \ge Y(\omega)$ for all $\omega \in\Omega$ such that $\mbP(\omega) > 0$.

\subsection{Risk Measures}
Following \cite{shapiro2021book}, given a finite probability space $(\Omega,\mcF,\mbP)$, a risk measure is a mapping $\rho:\mcL(\Omega)\to\mbR$. Consider the following properties:
\begin{enumerate}[label = (\alph*)]
    \item (Monotonicity) $\rho(X) \ge \rho(Y)$ for all $X,Y \in \mcL(\Omega)$ such that $X \ge Y$ a.s.;
    \item (Translation invariance) $\rho(X+\lambda) = \rho(X) + \lambda$ for any $\lambda \in \mbR$, $X \in\mcL(\Omega)$;
    \item (Normalization) $\rho(0) = 0$;
    \item (Convexity) For all $\alpha \in [0,1]$, $X, Y \in \mcL(\Omega)$, $\rho(\alpha X + (1-\alpha) Y) \le \alpha \rho(X) + (1-\alpha) \rho(Y)$;
    \item (Positive homogeneity) For all $\lambda \ge 0$, $X \in \mcL(\Omega)$, $\rho(\lambda X) = \lambda \rho (X)$.
\end{enumerate}
A risk measure is called monetary if it satisfies properties (a)-(c). It is called convex if it additionally satisfies (d), and coherent if it further satisfies (e).

We now introduce a special class of risk measures that possesses many desirable properties.
\begin{definition}\label{defn-UBSR}
    A risk measure on $(\Omega,\mcF, \mbP)$ is called a utility-based shortfall risk (UBSR) measure if it can be represented as:
    \begin{equation*}
        \SR(X) := \inf\left\{ m \in \mbR : \mbE[\ell(X-m)] \le 0 \right\},\quad \forall X\in \mcL(\Omega),
    \end{equation*}
    for some continuous strictly increasing loss function $\ell: \mbR \to \mbR$ such that $\ell(0)=0$.
\end{definition}
Note that the original definition of UBSR in \cite{Follmer2016book} is limited to convex loss functions. Here, we adopt the more general definition from \cite{Shen2014neuralcomp}, which allows for a broader class of loss functions and, consequently, a wider class of risk measures.

The following examples show that a lot of popular risk measures are special cases of UBSR.

\begin{example}[Expected value]
    The expectation corresponds to the UBSR induced by $\ell(x)=x$, which we refer to as the risk-neutral measure.
\end{example}

\begin{example}[Expectile]\label{exmp-expectile}
    The expectile is the only coherent UBSR \citep{Bellini2015qf}. It is induced by the loss function $\ell_{\mathrm{EXP}}(x):=\tau x^+-(1-\tau)x^-$, where $\tau\in[0,1]$ controls risk aversion. As $\tau$ increases from $0$ to $1$, the expectile ranges from the essential infimum to the essential supremum of the random variable, with the expectation recovered at $\tau=0.5$.
\end{example}

\begin{example}[Entropic risk measure]
    The ERM $\rho_{\mathrm{ERM}}(X) := \frac{1}{\beta} \log(\mbE[\mathrm{e}^{\beta X}])$ is a UBSR with loss function $\ell_{\mathrm{ERM}}(x) = \mathrm{e}^{\beta x}-1$, with $\beta > 0$ representing risk sensitivity.
\end{example}

\begin{example}[Soft quantile]\label{exmp-softquantile}
    The soft quantile risk measure introduced in \cite{Hau2025aistats} approximates the quantile risk measure while addressing the discontinuity and zero-slope issues of the quantile loss. It is defined through the loss function
    \begin{equation*}
    \begin{aligned}
        \ell_{\mathrm{SQ}}(x) := \begin{cases}
            (1-\mu) (\kappa x + \kappa^2 - 1), & x < -\kappa,\\
            \frac{1-\mu}{\kappa} x, & -\kappa \le x < 0,\\
            \frac{\mu}{\kappa} x, & 0\le x < \kappa,\\
            \mu (\kappa x - \kappa^2 + 1), & x \ge \kappa,
        \end{cases}
    \end{aligned}
    \end{equation*}
with $\mu \in (0,1)$ serving as the risk level and $\kappa > 0$ as the slope parameter. Unlike the quantile loss, the soft quantile loss is continuous and strictly increasing.
\end{example}

It is known that UBSR is elicitable, and it can be elicited by the scoring function $\psi(y-z)$ with $\psi(z) := \int_0^z \ell(\tau) d\tau$, (see Theorem 4.6, \cite{Bellini2015qf}), and $\SR(Y)$ is the unique minimizer of the expected scoring function, namely,
\begin{equation}\label{eq-UBSR-argmin}
    \SR(Y) = \argmin_{z \in \mbR} \mbE[\psi(Y-z)]. 
\end{equation}
By the fundamental theorem of calculus, $\psi(z)$ is differentiable and $\psi'(y) = \ell(y)$, for all $y \in \mbR$. From \cite{Emmer2015jor}, we know that if $\psi$ is the squared error, then the minimizer is the expectation.

As a summary for the properties of UBSR, we state as the following proposition.
\begin{proposition}[Theorem 4.113 in \cite{Follmer2016book}]\label{prop:UBSRProperty}
    For the UBSR measure defined in Definition \ref{defn-UBSR}, the following statements are equivalent: for any finite real-valued random variable $Y$, we have (i) $\SR(Y) = m^*$ and (ii) $\mbE[\ell(Y - m^*)] = 0$.
\end{proposition}

\subsection{Risk-aware MDPs}
Consider an MDP $(\mcX,\mcA,P,c,\gamma)$, where
$\mcX=\{1,\ldots,N\}$ and $\mcA$ are finite state and action spaces,
$P$ is the transition kernel, $c$ is the cost function, and
$\gamma\in(0,1)$ is the discount factor. A policy $\pi(a|x)$ specifies
the probability of selecting action $a$ in state $x$. We focus on non-anticipative policies with deterministic control, so that $\pi(\cdot|x)$ is a Dirac measure for each $x$. Under $\pi$, the induced
transition kernel is $P^\pi(x'|x):=\sum_{a\in\mcA}P(x'|x,a)\pi(a|x)$ for $x,x'\in \mcX$, and the induced cost function is $c^\pi(x):=\sum_{a\in\mcA}c(x,a)\pi(a|x)$ for $x\in \mcX$,
which is assumed bounded: $\max_{x\in\mcX}|c^\pi(x)|\le \bar{c}$.
Since we study policy evaluation under a fixed policy, we omit the superscript $\pi$ and write $P$ and $c$
for the policy-induced transition kernel and cost function.

In the risk-neutral policy evaluation setting, the objective is to evaluate the discounted cumulative cost $\bar{V}$ under a fixed policy $\pi$:
\begin{equation}\label{eq-RiskNeutralProblem}
    \bar{V}(x) := \mbE\left[\sum_{n=0}^{\infty} \gamma^n c(X_n) \mid X_0 = x \right], \quad \forall x \in \mcX,
\end{equation}
where $\gamma \in (0,1)$ is some discount factor and $\{X_n\}_{n\ge 0}$ represents the state trajectory, typically modeled as a Markov chain such that $X_{n+1} \sim P(\cdot|X_n)$ for $n \ge 0$.

To incorporate risk into an MDP, we adopt the dynamic risk measure of \cite{Ruszczynski2010mp}, which recursively evaluates future random costs through nested mappings. This structure ensures time consistency and yields a risk-aware Bellman equation for the optimal cumulative risk. To formulate the problem, we first introduce the notion of a risk map.
\begin{definition}[Definition 3.1, \cite{Shen2013sicon}]
    A mapping $\mcR:\mcX\times\mcL(\mcX) \to \mbR$ is said to be a risk map on the Markov chain $P$ if
    \begin{enumerate}[label = (\roman*)]
        \item for each $x\in\mcX$, $\mcR_x(\cdot) := \mcR(x,\cdot)$ is a risk measure;
        \item $\mcR(\cdot,Y) \in \mcL(\mcX)$ for each $Y \in \mcL(\mcX)$.
    \end{enumerate}
\end{definition}

The above definition guarantees that the risk map depends solely on the current state $x$, which is referred to as a Markovian risk map in \cite{Shen2013sicon}. Such a Markovian risk map ensures that the optimal policy remains stationary in infinite-horizon settings.

Throughout this work, we restrict our attention to the case where $\mcR_x(\cdot)$ is a UBSR measure. For notation simplicity, we denote it by $\SR_x(\cdot)$. Then for any fixed stage $T > 0$, the discounted $T$-stage total risk under policy $\pi$ is defined as
\begin{equation*}
    V_T(x) := c(x) + \gamma \SR_{x_0}\left(c(X_1) + \gamma \SR_{X_1}(c(X_2) + \cdots + \gamma \SR_{X_{T-1}}(c(X_T))\cdots)\right),\quad \forall x\in\mcX.
\end{equation*}
Letting $T \to \infty$, we obtain the infinite-horizon discounted total risk under policy $\pi$ as
\begin{equation}\label{eq-DiscountedTotalRiskObjective}
     V(x) := \lim_{T\to\infty} V_T(x),\quad \forall x\in\mcX.
\end{equation}
Under certain conditions, the limit in \eqref{eq-DiscountedTotalRiskObjective} exists for risk maps induced by the UBSR (see Lemma 5.3 and Theorem 5.5 in \cite{Shen2013sicon}). Furthermore, this results in the following risk-aware Bellman equation under policy $\pi$:
\begin{equation}\label{eq-UBSRBellman}
    V(x) = \SR_x(c(x) + \gamma V(X')),\quad \forall x\in\mcX,
\end{equation}
with $X'\sim P(\cdot|x)$, and $V$ as the unique solution.

\subsection{Policy Evaluation with Function Approximation in the Risk-neutral Case}

In practical RL problems, the state space may be too large to compute value estimates explicitly for each state in $\mcX$. A standard approach is to approximate the objective function $V$ using a linear function approximation based on a set of basis functions:
\begin{equation*}
    V_{\bmtheta}(x) := \bmtheta^\top \bmphi(x) = \sum_{i=1}^{d} \theta_i \phi_i(x),\quad \forall x\in\mcX,
\end{equation*}
where $\bmtheta \in \mbR^d$ is the parameter vector and $\bmphi: \mcX \to \mbR^d$ is a feature mapping. For each $i \in \{1,\ldots,d\}$, define $\bmphi_i := (\phi_i(1),\ldots,\phi_i(N))^\top \in \mbR^N$, which will be referred as basis vectors in $\mbR^N$. We assume that the collection $\{\bmphi_i\}_{i=1}^{d}$ is linearly independent, and let $\bmPhi := [\bmphi_1 \ \cdots \ \bmphi_d] \in \mbR^{N \times d}$ denote the feature matrix. When $d = N$, exact representation of $V$ over $\bmPhi$ is possible. In practice, however, computational considerations typically motivate the regime $d \ll N$, where $V$ is approximated via a low-dimensional linear representation $V_{\bmtheta}$.

For the risk-neutral formulation \eqref{eq-RiskNeutralProblem}, the value function $\bar{V}$ satisfies the Bellman equation
\begin{equation*}
    \mbE\left[\gamma \bar{V}(X') - \bar{V}(X) + c(X) \mid X = x\right] = 0, \quad \forall x \in \mcX,
\end{equation*}
where $X' \sim P(\cdot|x)$.
However, this identity generally cannot be satisfied when $\bar{V}$ is replaced by its linear approximation $V_{\bmtheta}$, since the Bellman operator does not preserve the linear function class. 

The classical TD(0) algorithm \citep{Sutton2018book} aims to find a parameter $\bar{\bmtheta}$ such that $V_{\bar{\bmtheta}}$ well approximates the value function $\bar{V}$. The update rule is given by
\begin{equation}\label{eq-RiskNeutralTD}
    \bmtheta_{n+1}= \bmtheta_n+ \eta_n \bmphi(X_n)
    \left[(\gamma \bmphi(X_{n+1}) - \bmphi(X_n))^\top \bmtheta_n+ c(X_n)\right], \quad n \ge 0,
\end{equation}
where $\eta_n$ is some step size and $\{X_n\}_{n\ge 0}$ is the state trajectory driven by the Markov chain $P$. 

TD(0) is known as a semi-gradient method \citep{Sutton2018book}. Its update direction is given by $g_n(\bmtheta) := (y_n - V_{\bmtheta}(X_n)) \nabla_{\bmtheta} V_{\bmtheta}(X_n)$, where $y_n := c(X_n) + \gamma V_{\bmtheta_n}(X_{n+1})$ is a sample-based estimate of the Bellman update for $V_{\bmtheta_n}$. The direction $g_n(\bmtheta_n)$ is often referred to as the negative semi-gradient of a certain squared loss function, since $g_n(\bmtheta_n) =- \nabla_{\bmtheta} \frac{1}{2}(y_n - V_{\bmtheta}(X_n))^2|_{\bmtheta=\bmtheta_n}$, where the target $y_n$ is treated as fixed, ignoring its implicit dependence on $\bmtheta_n$.
Suppose the Markov chain $P$ is ergodic with the vector of stationary distribution denoted as $q$, the iterates $\bmtheta_n$ converges almost surely to a parameter $\bar{\bmtheta}$ that is the solution to the following root-finding problem :
\begin{equation}\label{eq-SquareLossRoot}
    \mbE\left[\bmphi(X)\left((\gamma \bmphi(X') -\bmphi(X))^\top \bar{\bmtheta}+ c(X)\right)\right] = \bm{0},
\end{equation}
where $X \sim q$, and $X'\sim P(\cdot|X)$.

It turns out that, such $\bar{\bmtheta}$ is a fixed point for the projected Bellman equation and minimizes the mean squared projected Bellman error. 
We refer readers to \cite{Dann2014jmlr} for further discussion on policy evaluation algorithms in the risk-neutral case.

\section{UBSR-TD Algorithm}\label{sec:UBSRTD}
Motivated by the risk-neutral setting, we consider a linear function approximation of the risk-aware objective in \eqref{eq-DiscountedTotalRiskObjective}. Our goal is to develop an online algorithm for finding a parameter $\bmtheta^*$ such that $V_{\bmtheta^*}$ is a good approximation to $V$. In this section, we first present a stochastic value iteration algorithm for solving \eqref{eq-UBSRBellman}, then derive its linear function approximation based on the semi-gradient method.

Suppose $V$ is a solution to \eqref{eq-UBSRBellman}, then by Proposition \ref{prop:UBSRProperty}, we know that $V$ is a unique solution to the root finding problem
\begin{equation}\label{eq-UBSRRoot}
    \mbE\left[\ell \left(c(X) + \gamma V(X') - V(X)\right) \mid X=x \right] = 0,\quad \forall x\in\mcX.
\end{equation}
Following \cite{Shen2014neuralcomp}, it can be solved using the following stochastic approximation iteration:
\begin{equation}\label{eq-UBSRStochasticVI}
    V_{n+1}(x) = V_n(x) + \eta_n \ell \left(c(x) + \gamma V_n(X') - V_n(x)\right),\quad \forall x\in \mcX, X'\sim P(\cdot|x),
\end{equation}
where $\eta_n$ is some step size. Algorithm \eqref{eq-UBSRStochasticVI} can be interpreted as the policy-evaluation counterpart of the synchronous version of the Q-learning algorithm proposed in \cite{Shen2014neuralcomp}. Its almost sure convergence to the target value function $V$ follows from an analogous argument to that used in Theorem 3.2 of \cite{Shen2014neuralcomp}. An asynchronous variant can be derived similarly.

Algorithm \eqref{eq-UBSRStochasticVI} is value-function based and thus suffers from the curse of dimensionality when the state space is large. We next heuristically derive its linear function approximation counterpart, with a rigorous convergence analysis provided in Section \ref{sec:ConvergenceUBSRTD}.

From \eqref{eq-UBSR-argmin}, we replace the unknown value function $V$ in \eqref{eq-UBSRBellman} by its linear function approximation $V_{\bmtheta}$ and define the sample-based target $y := c(X)+\gamma V_{\bmtheta}(X')$.
Following the semi-gradient principle, we treat $y$ as fixed when differentiating with respect to $\bmtheta$. This leads to the surrogate minimization objective $\mbE[\psi(y - V_{\bmtheta}(X))]$,
where we average over the stationary distribution of $X$ to obtain a single
state-independent parameter $\bmtheta^*$.
Since $\psi$ is differentiable with derivative $\ell=\psi'$, the
corresponding first-order condition is
\begin{equation*}
    \mbE\left[\ell\left(y - V_{\bmtheta^*}(X)\right) \nabla_{\bmtheta} V_{\bmtheta}(X)|_{\bmtheta=\bmtheta^*}  \right]=\bm{0},
\end{equation*}
where $X\sim q$, and $X'\sim P(\cdot|X)$. For the linear
function approximation $V_{\bmtheta}(x)=\bmtheta^\top \bmphi(x)$, this condition becomes
\begin{equation}\label{eq-UBSR-TDRoot}
    \mbE\left[\bmphi(X)\ell\left((\gamma\bmphi(X')-\bmphi(X))^\top\bmtheta^* + c(X) \right)\right]=\bm{0}.
\end{equation}
This leads to the following stochastic approximation scheme: given an observed sequence of states $\{X_n\}_{n \ge 0}$, the parameter $\bmtheta_n$ is updated recursively according to
\begin{equation}\label{eq-UBSR-TD}
    \bmtheta_{n+1} = \bmtheta_n + \eta_n \bmphi(X_n) \ell\left( (\gamma\bmphi(X_{n+1}) - \bmphi(X_n))^\top \bmtheta_n + c(X_n)\right),\quad n\ge 0,
\end{equation}
with $\bmtheta_0$ arbitrarily initialized.
We call \eqref{eq-UBSR-TD} the \textit{UBSR temporal-difference (UBSR-TD) algorithm}. When $\ell$ is the identity function, which corresponds to $\psi$ being the square loss, the UBSR-TD algorithm \eqref{eq-UBSR-TD} reduces to the classical TD(0) with linear function approximation. 

From the derivation of UBSR-TD, the algorithm belongs to the semi-gradient method. Its update direction is $\tilde{g}_n(\bmtheta_n)=-\nabla_{\bmtheta}\psi(y_n - V_{\bmtheta}(X_n))|_{\bmtheta=\bmtheta_n}$. Clearly, $\tilde{g}_n$ serves as the general-loss counterpart of the TD(0) update direction $g_n$.

\begin{remark}
    Applying a loss function to the TD error is a common practice in RL to enhance robustness and training stability. Examples include the Huber loss (i.e., error clipping) \cite{Mnih2015nature, LiuWD2025aos} and the logit loss \cite{BasSerrano2021aistats}, among others.
    Formulations  \eqref{eq-UBSRStochasticVI} and \eqref{eq-UBSR-TD} show that passing the TD error through a non-linear loss function transforms the risk-neutral objective to a dynamic UBSR criterion.
\end{remark}

\section{Convergence for the UBSR-TD Algorithm}\label{sec:ConvergenceUBSRTD}
In this section, we study the convergence of UBSR-TD. Since the algorithm updates the feature parameter $\bmtheta$, we first introduce notation linking the value-function space to the feature space.

Before proceeding, we assume that the Markov chain $P$ is ergodic. This assumption is standard in the RL literature \citep{Tsitsiklis1997, Borkar2025aoap}.

\begin{assumption}\label{assump-StationaryDistribution}
    The Markov chain $P$ is finite, aperiodic and irreducible. Consequently, there exists a unique distribution $q$ such that $P^\top q = q$, and $\min_{x\in \mcX} q(x) > 0$. 
\end{assumption}

We now define the projection operator to the feature space. Consider the inner product $\langle \bmv, \bmw \rangle_q := \sum_{i=1}^N q(i) \bmv_i \bmw_i$ and the associated $q$-weighted norm $\norm{v}_q := \sqrt{\langle \bmv, \bmv\rangle_q}$, for any $\bmv,\bmw\in\mbR^N$, where $\bmv_i$ denotes the $i$-th component of $\bmv$. Let $\mathrm{span}(\bmPhi)$ denote the subspace spanned by the feature vectors. The orthogonal projection
$\Pi:\mcL(\mcX)\to \mathrm{span}(\bmPhi)$ under the $q$-weighted norm is defined by
\begin{equation*}
    \Pi(\bmv) := \argmin_{\bmz\in \mathrm{span}(\bmPhi)} \norm{\bmv-\bmz}_q.
\end{equation*}

Following the results in the risk-neutral case \citep{Tsitsiklis1997}, we also try to find an optimal $\bmtheta^*$ that is a good approximation for $V$. Here we define the ``good approximation" as the discrepancy $\norm{\bmPhi \bmtheta^* - V}_q$ can be controlled by $\norm{\Pi V - V}_q$ multiplied by some fixed constant.

We make the following assumptions that are essential for our analysis.

\begin{assumption}\label{assump-BoundedSlope}
    There exist $L_1,\epsilon_1 > 0$ such that $0 < \epsilon_1 \le \frac{\ell(x) - \ell(y)}{x-y}\le L_1$, for all $x\neq y \in \mbR$.
\end{assumption}

\begin{assumption}\label{assump-Convexity}
    The loss function $\ell$ is either convex or concave on $[0,\infty)$ and either convex or concave on $(-\infty,0)$.
\end{assumption}

\begin{assumption}\label{assump-phibounded}
    There exists a constant $M>0$ such that $\max_{x\in\mcX}\norm{\bmphi(x)}_2\le M$.
\end{assumption}

Assumption~\ref{assump-BoundedSlope} is standard in the UBSR literature \citep{Shen2014neuralcomp, Hu2018nrl, WangWK2025neurips} and is typically used to establish contraction of the associated risk-aware Bellman operators. The risk-neutral, expectile, and soft-quantile risk measures satisfy it by definition. Although ERM does not, a truncated variant can be constructed to satisfy the assumption (see Section A.2 in \cite{Shen2014neuralcomp}).

We next show some results regarding the stochastic value iteration \eqref{eq-UBSRStochasticVI}.  
Notice that \eqref{eq-UBSRStochasticVI} can be written equivalently into the following general stochastic approximation form:
\begin{equation}\label{eq-SA-FixedPoint}
    V_{n+1}(x) = V_n(x) + \frac{\eta_n}{\alpha} \left( (\mcH V_n)(x) - V_n(x) + M_{n+1}(x)\right),\quad \forall x\in\mcX,
\end{equation}
where $\alpha$ is some positive constant that will be specified later, $\mcH:\mcL(\mcX) \to \mcL(\mcX)$ is an operator and $M_{n+1}$ is a martingale difference sequence defined as follows: for all $x\in\mcX$,
\begin{equation}\label{eq-OperatormcH}
    \begin{aligned}
        (\mcH V_n) (x) &:= \alpha \sum_{x'\in\mcX} P(x'|x) \ell( \gamma V_n(x') - V_n(x) + c(x)) + V_n(x),\\
        M_{n+1}(x) &:= \alpha \ell(\gamma V_n(X') - V_n(x) + c(x)) - \alpha \sum_{x'\in\mcX} P(x'|x) \ell(\gamma V(x') - V(x) + c(x)),
    \end{aligned}
\end{equation}
where $X'\sim P(\cdot|x)$.
We have the following result regarding the operator $\mcH$.

\begin{lemma}\label{lemma:mcHProperties}
    Under Assumptions \ref{assump-StationaryDistribution} and \ref{assump-BoundedSlope}, let $\gamma < \epsilon_1/L_1$ and select $\alpha \in (0, \frac{2(\epsilon_1 - \gamma L_1)}{L_1^2(1+\gamma)^2})$. Then the operator $\mcH$ is an $\bar{\alpha}$-contraction on $(\mcL(\mcX), \norm{\cdot}_q)$ for some $\bar{\alpha} \in (0,1)$. Consequently, $\mcH$ admits a unique fixed point $V$, satisfying $V=\mcH V$. Furthermore, $V$ is a solution to \eqref{eq-UBSRRoot}.
\end{lemma}
\begin{remark}
    Under Assumptions \ref{assump-StationaryDistribution} and \ref{assump-BoundedSlope}, \cite{Shen2014neuralcomp} shows that, for any $\gamma\in(0,1)$, $\mcH$ is contractive under the infinity norm $\norm{\cdot}_\infty$ for some $\alpha>0$. Thus, \eqref{eq-UBSRRoot} has a unique solution, and \eqref{eq-UBSRStochasticVI} converges almost surely without further restrictions on $\gamma$.
    With function approximation, the analysis is restricted to $\mathrm{span}(\bmPhi)$, where $\Pi$ is non-expansive under $\norm{\cdot}_q$ but not necessarily under $\norm{\cdot}_\infty$. Contractivity of $\mcH$ under $\norm{\cdot}_q$ is therefore required, potentially imposing an additional restriction on $\gamma$. A similar condition appears in \cite{Kose2021jmlr} for the dynamic coherent risk setting. The role and necessity of this restriction are further illustrated in Example \ref{exmp:UBSRTDCounterExample}.
\end{remark}

Since $\Pi$ is non-expansive under $\norm{\cdot}_q$, the composition $\Pi\mcH$ is also an $\bar{\alpha}$-contraction under $\norm{\cdot}_q$. Because $\bmPhi$ has full rank, there exists a unique $\bmtheta^* \in \mbR^d$ such that $\bmPhi \bmtheta^* = \Pi\mcH (\bmPhi \bmtheta^*)$. Thus, $\bmPhi \bmtheta^*$ is the unique fixed point of the operator $\Pi\mcH$. We formally state the result as the following lemma.

\begin{lemma}\label{lemma:Phitheta-VstarBound}
    Under the assumptions of Lemma \ref{lemma:mcHProperties}, $\Pi \mcH$ is a contraction on $(\mcL(\mcX),\norm{\cdot}_q)$ and has a fixed point of the form $\bmPhi\bmtheta^*$ for a unique $\bmtheta^* \in \mbR^d$. Moreover, $\norm{\bmPhi \bmtheta^* - V}_q \le \frac{1}{1-\bar{\alpha}} \norm{\Pi V - V}_q$.
\end{lemma}

Clearly, if $d = N$, we have $\norm{\Pi V - V}_q = 0$ and Lemma \ref{lemma:Phitheta-VstarBound} implies that $\bmPhi \bmtheta^* = V$.

From \eqref{eq-UBSR-TD}, the UBSR-TD update can be written as
\begin{equation}\label{eq-UBSR-TD-SAwithMarkovNoise}
\begin{aligned}
    \frac{\alpha(\bmtheta_{n+1} - \bmtheta_n)}{\eta_n} = \alpha\bmphi(X_n) \ell\left( (\gamma \bmphi(X_{n+1}) - \bmphi(X_n))^\top \bmtheta_n + c(X_n) \right)
    =: H(\bmtheta_n, Y_{n+1}),
\end{aligned}
\end{equation}
where $Y_{n+1} := (X_{n}, X_{n+1})$, and $\alpha > 0$ is some constant. 
Define the filtration $\mcF_n := \sigma(\{\bmtheta_i, X_i\}_{i\le n})$. Under Assumption \ref{assump-StationaryDistribution}, it is straightforward to see that $\{Y_n\}_{n\ge 1}$ is also a Markov chain and is $\mcF_n$-measurable. Hence \eqref{eq-UBSR-TD-SAwithMarkovNoise} fits in the classic form of a stochastic approximation with Markovian noise \citep{ Borkar2025aoap, LiuSZ2025jmlr}. Then under proper conditions, the stochastic approximation \eqref{eq-UBSR-TD-SAwithMarkovNoise} converges almost surely to an equilibrium point of the mean-field ODE
\begin{equation}\label{eq-TD-ODE}
    \begin{aligned}
        \frac{d \bmvartheta(t)}{dt} &= h(\bmvartheta(t)) := \mbE_{Y \sim d_{Y}}[H(\bmvartheta(t), Y)]\\
        &= \alpha \sum_{x\in\mcX} q(x) \bmphi(x) \sum_{x'\in\mcX} P(x'|x) \ell\left( (\gamma \bmphi(x') - \bmphi(x))^\top \bmvartheta(t) + c(x) \right),
    \end{aligned}
\end{equation}
where $\bmvartheta(t)$ is the trajectory of the ODE at time $t$ in $\mbR^d$, and $d_{Y}$ is the stationary distribution of the Markov chain $\{Y_n\}_{n\ge 1}$ with $X_0 \sim q$ and $X_n \sim P(\cdot|X_{n-1})$ for $n\ge 1$.
If the ODE \eqref{eq-TD-ODE} has an equilibrium point $\tilde{\bmtheta}$, then we have
\begin{equation}\label{eq-ODE-equilibrium}
    \bm{0} = \sum_{x\in\mcX} q(x) \bmphi(x) \sum_{x'\in\mcX} P(x'|x) \ell\left( (\gamma \bmphi(x') - \bmphi(x))^\top \tilde{\bmtheta} + c(x) \right),
\end{equation}
which is exactly the root finding problem \eqref{eq-UBSR-TDRoot}.

It turns out that the conditions in Lemma \ref{lemma:mcHProperties} ensure that the ODE \eqref{eq-TD-ODE} has a unique equilibrium, which coincides with the optimal $\bmtheta^*$ defined in Lemma \ref{lemma:Phitheta-VstarBound}.

\begin{lemma}\label{lemma:ODEUniqueEquilibrium}
    Under the assumptions of Lemma \ref{lemma:mcHProperties}, the ODE \eqref{eq-TD-ODE} has a unique equilibrium point $\bmtheta^*$, with $\bmPhi \bmtheta^*$ being the unique fixed point for the operator $\Pi \mcH$.
\end{lemma}

We now turn to establish the almost sure convergence of the UBSR-TD algorithm \eqref{eq-UBSR-TD}. To do this, define an operator $\hat{\mcH}: \mbR^d \to \mbR^d$ as
\begin{equation}\label{eq-OperatorhatH}
    \hat{\mcH} (\bmtheta):= \alpha \sum_{x\in\mcX} q(x) \bmphi(x) \sum_{x'\in\mcX} P(x'|x) \ell\left( (\gamma \bmphi(x') - \bmphi(x))^\top \bmtheta + c(x) \right) + \bmtheta,\quad \forall \bmtheta\in\mbR^d.
\end{equation}
Then the ODE \eqref{eq-TD-ODE} can be written as $\frac{d\bmvartheta(t)}{dt} = \hat{\mcH} (\bmvartheta(t)) - \bmvartheta(t)$ and clearly the equilibrium point $\bmtheta^*$ of the ODE \eqref{eq-TD-ODE} is a fixed point for the operator $\hat{\mcH}$, i.e., $\hat{\mcH} (\bmtheta^*) = \bmtheta^*$, which is exactly a solution to the root finding problem \eqref{eq-UBSR-TDRoot}.

We make the following assumption on the iterative step size.
\begin{assumption}\label{assump-StepSizeLimit}
    The step size $\{\eta_n\}_{n=0}^\infty$ satisfies $\sum_{n=0}^\infty \eta_n = \infty$ and $\sum_{n=0}^\infty \eta_n^2 < \infty$.
    Furthermore, the limit $\lim_{n \to \infty} \left( \eta_{n+1}^{-1} - \eta_n^{-1} \right) =: \eta^* $exists and is finite.
\end{assumption}
This assumption is somewhat stronger than the classical Robbins-Monro condition and is required for technical reasons (see \cite{Borkar2025aoap} and \cite{LiuSZ2025jmlr} for details). Commonly used choices, such as $\eta_n = b_1/(n + b_2)^{\delta}$ with $b_1,b_2 > 0, \delta\in (1/2,1]$, satisfy this requirement.

We are ready to present the almost sure convergence result for the UBSR-TD algorithm.
\begin{theorem}\label{thm:UBSRTDASConvergence}
    Suppose that Assumptions \ref{assump-BoundedSlope}-\ref{assump-StepSizeLimit} hold and that $\gamma < \epsilon_1/L_1$. Then, the UBSR-TD algorithm \eqref{eq-UBSR-TD} converges almost surely to $\bmtheta^*$, where $\bmtheta^*$ is defined as in Lemma \ref{lemma:Phitheta-VstarBound}.
\end{theorem}

Theorem~\ref{thm:UBSRTDASConvergence} establishes that UBSR-TD converges almost surely to a unique $\bmtheta^*$, where $\Phi\bmtheta^*$ is a fixed point of the projected risk-aware Bellman operator $\Pi\mcH$. Lemma~\ref{lemma:ODEUniqueEquilibrium} then indicates that such $\bmtheta^*$ is a unique equilibrium of the ODE \eqref{eq-TD-ODE} and hence a solution to the root-finding problem \eqref{eq-UBSR-TDRoot}.

We comment on the proof of Theorem \ref{thm:UBSRTDASConvergence}. Since the operator $\hat{\mcH}$ is nonlinear in $\bmtheta$, the classical linear stochastic approximation result of \cite{Tsitsiklis1997} is not directly applicable. We instead invoke the recent stochastic approximation framework with Markovian noise developed in \cite{Borkar2025aoap} to establish the almost sure convergence. 
Although aperiodicity of the Markov chain is likely unnecessary to establish the almost sure convergence, as noted in \cite{Tsitsiklis1997} for TD(0), we impose this assumption because our proof relies on Theorem 1 of \cite{Borkar2025aoap}, where aperiodicity is used to establish geometric ergodicity of the Markov chain. Alternative techniques, such as those in \cite{LiuSZ2025jmlr}, may establish convergence under weaker conditions and general state space, but would require a more involved analysis, as Assumption 3 of \cite{LiuSZ2025jmlr} generally does not hold in our setting.

The condition $\gamma < \epsilon_1/L_1$ is essential for the well-posedness and convergence of UBSR-TD, ensuring convergence to a unique stationary point of the ODE \eqref{eq-TD-ODE}. When it is violated, the limiting ODE \eqref{eq-TD-ODE} may admit multiple or unstable equilibria, as shown in the following example.

\begin{example}\label{exmp:UBSRTDCounterExample}
    Consider a zero-cost, two-state Markov chain with symmetric transitions $P(x_i|x_j) = 1- P(x_i | x_i) = p$ for $i\neq j\in\{1,2\}$. 
    Let the feature dimension be one with $\bmPhi=[1\ \delta]^\top$ and $\delta=\frac{891\pm \sqrt{465917}}{182}$. 
    Consider the expectile risk with loss function $\ell(y)= L_1 y$ for $y\ge 0$ and $\ell(y)=\epsilon_1 y$ for $y<0$. 
    We set $\epsilon_1 = 1 - L_1 = 1/11$ and $\gamma=9/10>\epsilon_1/L_1=1/10$, violating the conditions of Theorem~\ref{thm:UBSRTDASConvergence}.
    Nevertheless, one can easily identify the true value function $V(x)=0$ for all $x$, exactly represented by $\bmtheta^*=0$. 
    When $p=9/10$, Equation \eqref{eq-UBSR-TDRoot} holds for every $\bmtheta\geq 0$. Consequently, the ODE \eqref{eq-TD-ODE} has a continuum of equilibria because $\frac{d\bmvartheta(t)}{dt}=0$ for $\bmvartheta(t)\ge 0$. Therefore, within this region, the UBSR-TD iterates have zero expected drift. 
    When $p\in(9/10,1)$, the ODE \eqref{eq-TD-ODE} has a unique equilibrium at $\bmtheta^*=0$, which is unstable from the positive side, since $\frac{d \bmvartheta(t)}{dt}>0$ whenever $\bmvartheta(t)>0$. Consequently, for $\bmtheta_0>0$, the UBSR-TD iterates exhibit a positive expected drift away from zero.

    The bounds on $\gamma$ can be tightened by incorporating transition and feature information. For example, \cite{Kose2021jmlr} derives a sufficient condition $\gamma < 1/\sqrt{1+\chi}$ using the distortion coefficient $\chi$ for coherent risk measures. In our example, this gives $\gamma < \sqrt{19}/10$ for the expectile when $p=9/10$, improving on $1/10$ but remaining below the true threshold $9/10$. Therefore, the condition $\gamma < \epsilon_1/L_1$ should be viewed as a robust sufficient condition for general UBSR satisfying Assumption~\ref{assump-BoundedSlope}, with sharper bounds available under problem-specific structures.
\end{example}

\section{Extensions}\label{sec:Extensions}
TD-type algorithms are known to converge slowly when the step size is not carefully tuned \citep{Sutton2018book}. To address this issue, we introduce two extensions: UBSR-TD($\lambda$) and UBSR-Newton. These algorithms possess distinct features and, as shown in our experiments, can outperform UBSR-TD in certain settings. We present them to highlight that existing risk-neutral policy evaluation methods can be adapted to the dynamic UBSR objective by simply incorporating the loss function to the TD error.
The algorithm pseudo-codes are provided in Appendix \ref{sec:AlgorithmPseduoCodes}.

\subsection{UBSR-TD(\texorpdfstring{$\lambda$}{lambda}) Algorithm}

Eligibility traces can accelerate learning by incorporating temporal-difference information over multiple time steps \citep{Dann2014jmlr}. By interpolating between TD learning and Monte-Carlo sampling, they are regarded as one of the most natural extensions of TD-type methods \citep{Sutton2018book}. In the following, we incorporate eligibility traces into UBSR-TD.

Let $\lambda \in [0,1]$. Given an observation of the streaming data $\{X_n\}_{n\ge 0}$. Define a sequence of eligibility vector $\bmzeta_n \in \mbR^d$ as $\bmzeta_n := \sum_{k=0}^n (\gamma \lambda)^{n-k} \bmphi(X_k)$, which can be recursively computed by
\begin{equation}\label{eq-UBSR-TDLambda-EligibilityTrace}
     \bmzeta_{n+1} = \gamma \lambda \bmzeta_n + \bmphi(X_{n+1}),
\end{equation}
with $\bmzeta_0 := \bmphi(X_0)$.
The \textit{UBSR-TD($\lambda$) algorithm} is defined as 
\begin{equation}\label{eq-UBSR-TDLambda}
    \bmtheta_{n+1} = \bmtheta_n + \eta_n \bmzeta_n \ell\left((\gamma\bmphi(X_{n+1}) - \bmphi(X_n))^\top \bmtheta_n + c(X_n)\right),
\end{equation}

When $\lambda=0$, the algorithm reduces to UBSR-TD. For $\lambda>0$, however, UBSR-TD($\lambda$) need not converge to the same objective \eqref{eq-UBSR-TDRoot} as UBSR-TD, unlike its counterpart in the risk-neutral setting. This phenomenon is demonstrated in our numerical experiments (see Section \ref{sec:ConvergenceExperimentsd5}). \cite{Kose2021jmlr} (Section 5) develops a related eligibility-trace-based TD-type algorithm for MDPs with dynamic coherent risk measures. Their analysis shows that, in the dynamic risk setting, TD($\lambda$)-type algorithms generally solve a different root-finding problem, with the original objective distorted by a $\lambda$-dependent weight matrix.
Moreover, as shown in our experiments, UBSR-TD(1) generally does not minimize the projected value-function error, in contrast to classical TD(1). 

\subsection{UBSR-Newton Algorithm}
Newton-type methods have been widely used to accelerate TD-type algorithms \citep{Givchi2014acml, PanYC2017aaai}. By leveraging second-order information, they are shown to improve convergence in risk-neutral settings. We next develop an analogous approach for the UBSR framework. Throughout this subsection, we assume that $\ell$ is differentiable.
For notation simplicity, let $\bmX_n := \bmphi(X_n)$, $\bmZ_{n} := \gamma\bmphi(X_{n+1}) - \bmphi(X_n)$, $c_n := c(X_n)$. Then $(\bmX_n,\bmZ_n, c_n)$ represents a one-stage observation generated by the Markov chain $P$. 

For the risk-neutral problem \eqref{eq-RiskNeutralProblem}, \cite{LiuWD2025aos} use the full observation history to update both the second-order information matrix and the improvement direction. Extending this idea to the UBSR setting, given observations ${(\bmX_i,\bmZ_i,c_i)}_{i\ge 0}$, we recursively update
\begin{equation}\label{eq-UBSRNewtonAlgorithmOrigin}
    \hat{\bmtheta}_{n+1} = \frac{1}{n+1} \sum_{i=1}^{n+1} \hat{\bmtheta}_{i-1} - \hat{\bmH}_{n+1}^{-1} \frac{1}{n+1} \sum_{i=1}^{n+1} \bmX_{i} \ell\left(\bmZ_{i}^\top \hat{\bmtheta}_{i-1} + c_{i}\right),
\end{equation}
where $\hat{\bmH}_{n+1}$ is an empirical information matrix for \eqref{eq-UBSR-TDRoot} defined as
\begin{equation}\label{eq-hatH}
    \hat{\bmH}_{n+1} := \frac{1}{n+1} \sum_{i=1}^{n+1} \bmX_{i} \bmZ_{i}^\top \ell'\left( \bmZ_{i}^\top \hat{\bmtheta}_{i-1} + c_{i} \right).
\end{equation}

As suggested in \cite{LiuWD2025aos}, Whenever $\hat{\bmH}_n$ is invertable, the formulation in \eqref{eq-UBSRNewtonAlgorithmOrigin}-\eqref{eq-hatH} can be rewritten in a fully online form by introducing two auxiliary variables and applying the Sherman-Morrison formula. Consequently, the resulting update scheme requires no more than $O(d^2)$ additional computational and memory cost per iteration. Moreover, to ensure the invertibility of an initial $\hat{\bmH}_0$, one can employ a ``warm-up'' procedure: we use the first $n_0$ samples to construct an invertible $\hat{\bmH}_{n_0}$ before carrying out the Newton update.

We summarize the modified iteration procedure as follows. First, we initialize $\hat{\bmtheta}_0$ and use the initial $n_0$ samples to compute:
\begin{equation}\label{eq-PretrainHandL}
    \hat{\bmH}_{n_0} := \frac{1}{n_0} \sum_{i=1}^{n_0} \bmX_i \bmZ_i^\top \ell'(\bmZ_i^\top \hat{\bmtheta}_0 + c_i),\quad 
    \bmL_{n_0} := \frac{1}{n_0} \sum_{i=1}^{n_0} \bmX_i \ell(\bmZ_i^\top \hat{\bmtheta}_0 + c_i),
\end{equation}
ensuring that $\hat{\bmH}_{n_0}$ is invertable. Note that for $\hat{\bmH}_{n_0}$ to be full rank, it is necessary to have $n_0 \ge d$. Let $\hat{\bmtheta}_i = \bar{\bmtheta}_i := \hat{\bmtheta}_0$ for $i = 1,\cdots, n_0$. Then starting from $n = n_0$, we recursively compute
\begin{equation}\label{eq-hattheta=bartheta-hatHinverseL}
    \hat{\bmtheta}_{n+1} = \bar{\bmtheta}_{n+1} - \hat{\bmH}_{n+1}^{-1} \bmL_{n+1},
\end{equation}
where
\begin{equation}\label{eq-barthetaIteration}
     \bar{\bmtheta}_{n+1} := \frac{1}{n+1} \left( n \bar{\bmtheta}_n + \hat{\bmtheta}_n \right),
\end{equation}
\begin{equation}\label{eq-LIteration}
\begin{aligned}
    \bmL_{n+1} := 
     \frac{n}{n+1} \bmL_n + \frac{1}{n+1} \bmX_{n+1} \ell\left(\bmZ_{n+1}^\top \hat{\bmtheta}_n + c_{n+1}\right), 
\end{aligned}
\end{equation}
\begin{equation}\label{eq-hatHn}
    \hat{\bmH}_{n+1} = \frac{n}{n+1} \hat{\bmH}_n + \frac{1}{n+1}
    \bmX_{n+1} \bmZ_{n+1}^\top \ell'\left(\bmZ_{n+1}^\top \hat{\bmtheta}_{n} + c_{n+1} \right).
\end{equation}
Once $\hat{\bmH}_{n}^{-1}$ is available, $\hat{\bmH}_{n+1}^{-1}$ can be computed recursively using the Sherman-Morrison formula:
\begin{equation}\label{eq-UBSRNewton-ShermanMorrison}
    \hat{\bmH}_{n+1}^{-1} = \frac{n+1}{n} \hat{\bmH}_n^{-1} 
    - \frac{\frac{n+1}{n^2} \hat{\bmH}_n^{-1} \bmX_{n+1} \bmZ_{n+1}^\top \hat{\bmH}_n^{-1}}{(\ell'(\bmZ_{n+1}^\top\hat{\bmtheta}_n + c_{n+1}))^{-1} + \frac{1}{n} \bmZ_{n+1}^\top \hat{\bmH}_{n}^{-1} \bmX_{n+1}}.
\end{equation}
We call the update \eqref{eq-PretrainHandL}-\eqref{eq-UBSRNewton-ShermanMorrison} the \textit{UBSR-Newton algorithm}.

The UBSR-Newton algorithm extends the ROPE method of \cite{LiuWD2025aos} for online robust policy evaluation, in which $\ell$ is chosen as the derivative of a smoothed Huber loss. Our formulation generalizes ROPE to a broader class of Newton-type algorithms, and the resulting algorithm can be interpreted as policy evaluation in a risk-aware MDP framework with dynamic UBSR measures.

Although almost sure convergence has not been established for either ROPE or UBSR-Newton, techniques analogous to those in \cite{LiuWD2025aos} may be used to derive uniform high-probability bounds and variance estimators for UBSR-Newton.

\section{Experiments}\label{sec:Experiments}

In this section, we conduct numerical experiments to evaluate the algorithms proposed in the preceding sections. We first examine the algorithms in a tabular MDP setting and validate the convergence result established in Theorem \ref{thm:UBSRTDASConvergence}. We then further demonstrate the potential of our algorithms through a real-world risk-aware perishable inventory problem, where function approximation is required to address the curse of dimensionality.

\subsection{Convergence Experiments for Different Algorithms}\label{sec:ExperimentTabular}
We test the policy evaluation algorithms on a randomly generated MDP with $N=10$ states and 10 actions per state.
The nominal transition kernel is sampled uniformly on $[0,1]$ and normalized, the cost function is drawn from the normal distribution $\mcN(10,100)$, and the deterministic policy is generated randomly. The feature matrix is orthogonalized via QR decomposition to ensure full rank, The initial $\bmtheta_0$ is set to $\bm{0}$. 
For UBSR-TD, we use $\eta_n=2/(n+100)^{2/3}$, a value calibrated in preliminary experiments to promote convergence.
We set $n_0 = 500$ for UBSR-Newton, and the resulting $\hat{\bmH}_{n_0}$ is observed to be invertable in all the test scenarios.

We consider two representative UBSR measures: the expectile and the soft quantile. For the expectile, Theorem~\ref{thm:UBSRTDASConvergence} requires $\gamma < \gamma_{\mathrm{EXP}} := \min\{ \frac{\tau}{1-\tau}, \frac{1-\tau}{\tau}\}$. For the soft quantile, we set $\kappa = 2$ and vary $\mu \in (0,1)$. It is worth noting that, under these parameter choices, the soft quantile is neither convex nor concave over its entire domain, yet it still satisfies Assumption \ref{assump-Convexity}.
Since both the loss functions for the expectile and soft quantile are piecewise linear and non-differentiable at kink points, we use the left derivatives at those points for UBSR-Newton. Due to space limitation, we report the results for the soft quantile in Appendix \ref{sec:ConvergenceSoftQuantile}.

The algorithm's performance is measured by the $q$-weighted distance $\norm{\bmPhi\bmtheta_n-\Pi V}_q$ along the simulated trajectory, where $V$ is computed by risk-aware value iteration (see Section 7 of \cite{Ruszczynski2010mp}). Each algorithm is run independently 20 times for $10^6$ iterations, and the distance is averaged across the 20 trajectories at each iteration.

\subsubsection{Convergence Experiments for $d = N$}\label{sec:ConvergenceExperiementsd10}
The first two rows of Figure \ref{figure:ConvergenceExpectile} show the convergence of UBSR-TD, UBSR-TD($\lambda$) and UBSR-Newton (labeled as TD($0.0$), TD($\lambda$), Newton) for the expectile when $d=N$.  The shaded regions represent the standard error of the first component of $\bmtheta_n$ ($\hat{\bmtheta}_n$ for UBSR-Newton) across 20 trajectories. Under this setting, the features can recover the true value function. The proposed algorithms are observed to converge to the true value function around $10^6$ iterations. In case of UBSR-TD, this occurs even when the discount-factor restriction in Theorem \ref{thm:UBSRTDASConvergence} is not satisfied.
We can see that, UBSR-Newton requires substantially fewer iterations than UBSR-TD to reach a prescribed error threshold relative to the true value function. In most settings, UBSR-TD($\lambda$) converges faster than UBSR-TD, with larger values of $\lambda$ generally improving convergence, particularly for larger discount factors. 
The standard error is relatively large in early iterations and gradually decreases as the algorithms converge. As in the risk-neutral setting, the variance of UBSR-TD($\lambda$) increases with $\lambda$ and UBSR-Newton exhibits large variance during the early iterations possible due to the instability of $\hat{\bmH}_n$.

\subsubsection{Convergence Experiments for $d < N$}\label{sec:ConvergenceExperimentsd5}

The last two rows of Figure~\ref{figure:ConvergenceExpectile} report results for expectile with $d=5<N$, where the feature space cannot represent the true value function. When the restriction on $\gamma$ holds, the proposed algorithms converge, but not necessarily to the projected value function.
As established in Theorem~\ref{thm:UBSRTDASConvergence}, UBSR-TD converges to the solution of \eqref{eq-UBSR-TDRoot}. Therefore, $\norm{\bmPhi\bmtheta_n-\Pi V}_q$ need not vanish. UBSR-Newton exhibits similar behavior. When the restriction on $\gamma$ fails, UBSR-TD($\lambda$) may diverge, as observed for $\tau\in\{0.6,0.9\}$ and $\gamma=0.9$.

For expectile with $\tau = 0.5$, UBSR-TD(1) reduces to TD(1), for which $\norm{\bmPhi\bmtheta_n - \Pi V}_q$ is expected to converge to zero when $d < N$ \citep{Tsitsiklis1997}. However, this property generally fails in the risk-aware setting. Moreover, UBSR-TD(1) need not achieve the smallest distance to the projected value function, as illustrated by the case where $\tau = 0.6$ and $\gamma = 0.6$.

Overall, the experiments demonstrate the proposed algorithms’ ability to handle a broad range of risks and discount factors. For practical applications, when the feature representation is imperfect, UBSR-TD and UBSR-Newton are practical choices: UBSR-TD is better suited to low-variance estimation, while UBSR-Newton is preferred for faster convergence.

\begin{figure}[htbp]
    \centering
    \includegraphics[width=0.95\linewidth]{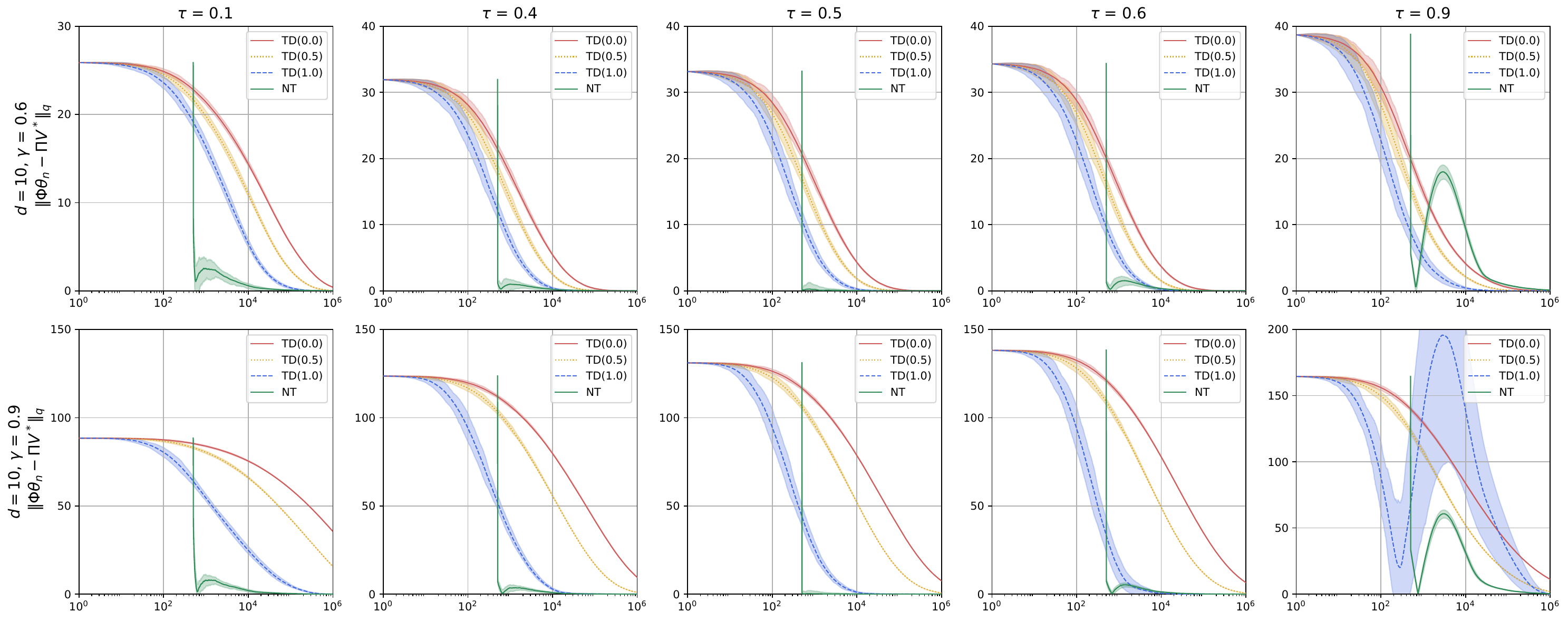}
    \includegraphics[width=0.95\linewidth]{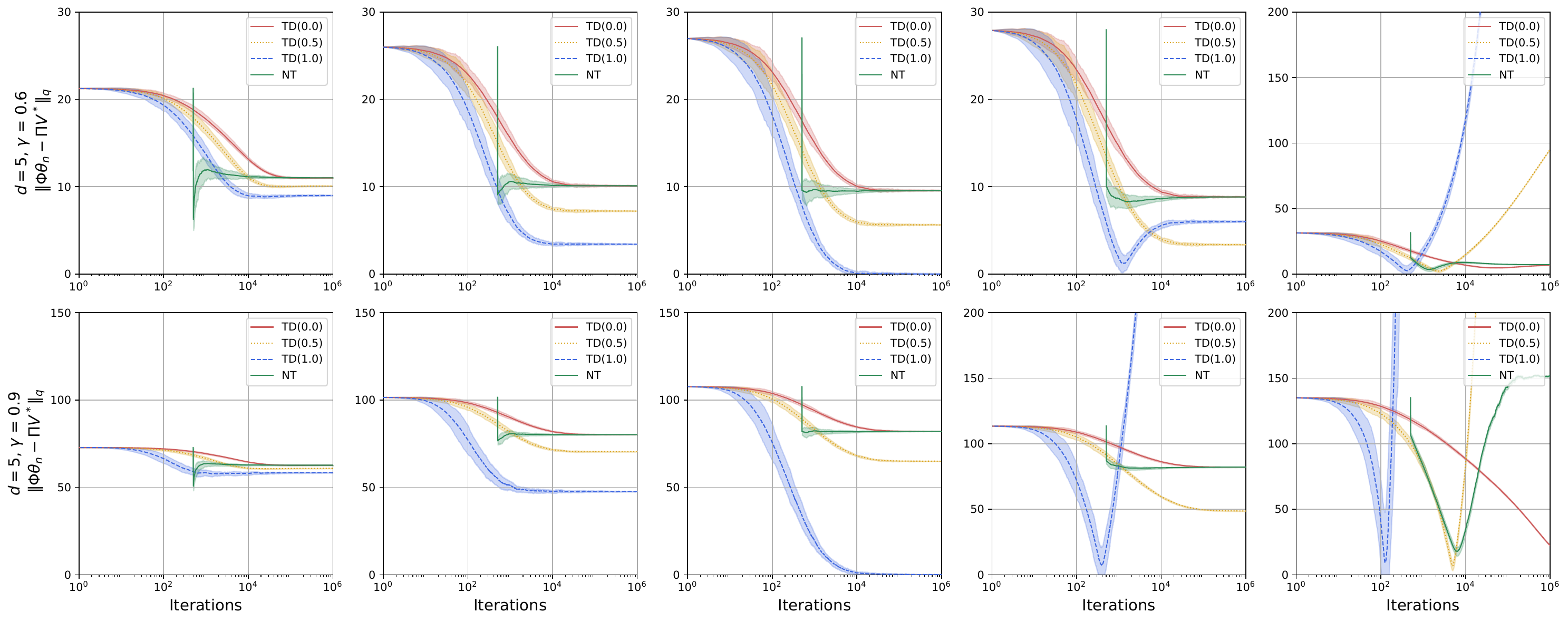}
    \caption{Convergence results for expectile risk}
    \label{figure:ConvergenceExpectile}
\end{figure}

\subsection{Experiments on a Perishable Inventory Management Problem}\label{sec:PerishableInventoryProblem}
Inventory management problems have been studied extensively. When products are non-perishable or have deterministic shelf lives, DP can often be used to compute optimal policies efficiently \citep{Nahmias2011book}. By contrast, stochastic shelf lives typically destroy classical structural properties, substantially increasing problem complexity and making standard heuristics, such as $(s,S)$ policies, generally suboptimal (see Section 4 of \cite{Abouee2026ijoc}). In large-scale settings, exact DP also becomes computationally prohibitive because of the curse of dimensionality. A common solution is approximate dynamic programming (ADP), in which the value function is approximated using low-dimensional basis features. Policy iteration is then carried out by alternating between
policy improvement and policy evaluation steps \citep{Powell2011book}.

Stochastic shelf lives further motivate risk-sensitive decision-making in high-stakes applications such as medical resource management \citep{FengYY2008opre,HosseiniMotlagh2019jaihc}. Larger orders can improve service levels but increase the risk of wastage, whereas smaller orders reduce wastage but may result in severe shortages with significant medical consequences. 
Motivated by these challenges, in this subsection, we propose UBSR-based risk-aware policy iteration algorithms within the ADP framework and apply it to the perishable platelet inventory system studied in \cite{Abouee2026ijoc}. 

\subsubsection{Model Formulation}
The system state is defined as $\bmX=(X_1,\cdots,X_{m-1})\in\mcX\subset\mbZ_+^{m-1}$, where $X_i$ is the inventory level with $i$ periods of remaining shelf life and $m$ is the maximum shelf life. After observing the state, the decision-maker orders $z\in\mcZ:=\{0,\cdots,\bar{z}\}$. The delivered units have random remaining shelf lives $\bmY=(Y_1,\cdots,Y_m)$, with $\sum_{i=1}^m Y_i=z$. Conditional on $z$, each unit independently has shelf life $i$ with probability $p_i(z)$ under a multinomial logit model. The random demand $D$ follows a truncated negative binomial distribution. See Section 6.1 of \cite{Abouee2026ijoc} for details on the shelf-life and demand distributions.

An admissible policy $\mu:\mcX\to\mcZ$ specifies the order quantity for each state. In each period, the decision-maker observes the inventory state, places an order, and receives the delivery immediately. Demand is subsequently realized and satisfied according to a first-expire-first-out policy. Unmet demand is lost, and each remaining unit loses one period of shelf life when the period ends.

Given a state $\bmx=(x_1,\dots,x_{m-1}) \in \mcX$ and policy $\mu$, the next state $\bmX=(X_1,\dots,X_{m-1})$ evolves according to
\begin{equation}\label{eq-ADPDynamic}
\begin{aligned}
    X_j &:= \Bigg( x_{j+1} + Y_{j+1} - \Big(D - \sum_{i=1}^j (x_i + Y_i)\Big)^+ \Bigg)^+, \quad j = 1, \dots, m-2,\\
    X_{m-1} &:= \Bigg( Y_m - \Big(D - \sum_{i=1}^{m-1} (x_i + Y_i)\Big)^+ \Bigg)^+,
\end{aligned}
\end{equation}
where $\bmY=(Y_1,\dots,Y_m)$ denotes the random vector of received products under the order $\mu(\bm{x})$.

The one-stage cost function is defined as
\begin{equation}\label{eq-ADPcost}
\begin{aligned}
    c(\bmX, \mu(\bmX)) 
    := & c_1 \bm{1}_{\{\mu(\bmX) > 0\}} 
    +c_2 \left( \sum_{i=1}^{m-1} X_i + \mu(\bmX) - D \right)^+ \\
    &+ c_3 \left( D - \sum_{i=1}^{m-1} X_i - \mu(\bmX) \right)^+  + c_4 \left( X_1 + Y_1 - D \right)^+,
\end{aligned}
\end{equation}
where $c_1$, $c_2$, $c_3$, $c_4$ represent the fixed ordering, holding, shortage, and wastage costs, respectively.

Our objective is to determine an optimal policy minimizing the infinite-horizon discounted total risk \eqref{eq-DiscountedTotalRiskObjective} associated with the cost function \eqref{eq-ADPcost} under the dynamic expectile risk.

\subsubsection{Approximate Policy Iteration}\label{sec:ApproximatePolicyIteration}
For a maximum inventory level of $20$, the perishable inventory system has $\binom{19+m}{m-1}=\frac{m(m+1)\cdots (m+19)}{20!}$ states. The state-space size grows polynomially as $m$ increases, making explicit value iteration computationally expensive.

We employ a linear function approximation of the value function: $\tilde{V}(\bmx) := \sum_{i=1}^d \theta_i \phi_i(\bmx)$, $\bmx \in \mcX$, where $\{\phi_i\}_{i=1}^d$ are predefined basis functions. 
Following \cite{Abouee2026ijoc}, we set $\phi_1(\bmx)=1$ and $\phi_2(\bm{x}) = \hat{V}(\sum_{i=1}^{m-1}x_i)$, where $\hat{V}$ is the value function of the associated risk-aware non-perishable inventory problem, which can be computed efficiently via value iteration. The remaining features capture the additional value structure induced by stochastic shelf-life dynamics that is not represented by $\hat{V}$.
We adopt the polynomial feature class Choice 4 of the state variables in \cite{Abouee2026ijoc}, comprising $x_i$, $x_i^2$, and $x_i x_j$ for $i,j=1,\ldots,m-1$ with $i\neq j$.
All features are normalized to improve numerical stability and avoid computational scaling issues.

Starting from an initial state $\bmx$, the policy improvement step at iteration $n$ is defined by
\begin{equation*}
\mu_n(\bmx) := \argmin_{z\in\mcZ} \SR\left[
c(\bmx,z) + \gamma \sum_{i=1}^d \theta_{n,i}^*\phi_i(\bmX^{\mu_{n-1}}) \right],
\end{equation*}
where $\mu_n$ and $\{\theta_{n,i}^*\}_{i=1}^d$ are the policy and weight parameters at iteration $n$; $\bmX^{\mu_{n-1}}$ denotes the random next state generated by the transition dynamics in \eqref{eq-ADPDynamic} under policy $\mu_{n-1}$; and $\gamma \in (0,1)$ is the discount factor.
When the transition dynamic is unknown or exact policy improvement is costly, we use simulation or heuristics for policy improvement with a fixed number of samples and and run UBSR-TD or UBSR-Newton for a fixed number of policy-evaluation iterations.

For benchmark comparisons, we consider two policies: a \textit{static} policy derived from the non-perishable model, corresponding to the classical $(s,S)$ policy, and a \textit{myopic} policy that directly minimizes the risk associated with the one-stage cost.

\subsubsection{Experimental Setup}
We consider a perishable inventory system with maximum inventory $20$, shelf life $m\in\{3,5,8\}$, and discount factor $\gamma=0.6$, which satisfies the UBSR-TD convergence conditions in Theorem~\ref{thm:UBSRTDASConvergence} for $\tau\in(3/8,5/8)$. Demand follows the negative binomial distribution in Table 1 of \cite{Abouee2026ijoc}, with mean $6.2$. Shelf-life parameters are $(1,0.5,-0.2,-0.1)$ for $m=3$ and the endogenous settings in Table 2 of \cite{Abouee2026ijoc} for $m=5,8$, where the negative coefficients imply that larger orders increase the probability of receiving items with lower period of remaining shelf life.

We collect 60,000 online samples for policy evaluation, divide them into 20 trajectories of length 3,000, and average the resulting weight estimates. For UBSR-Newton, we use the first 1,000 samples of each trajectory for warm-up, which yields an invertible $\hat{\bmH}_{n_0}$ in most cases. Trajectories with singular $\hat{\bmH}_{n_0}$ are discarded. Empirically, this is observed to have negligible impact on the results.

For $m=3$, exact policy improvement is computationally negligible. For $m=5$, however, each exact improvement step takes hours. We therefore perform one exact step to obtain a myopic baseline, followed by approximate improvement using 50 simulated transitions per state. This reduces runtime from hours to minutes while preserving reasonable empirical performance. We run policy iteration for 20 iterations.
For $m=8$, exact policy evaluation takes over 10 hours, making exact policy comparisons impractical. We therefore only compare the runtime and optimality gap of UBSR-TD and UBSR-Newton for a single policy evaluation step for $m=8$. The runtime and optimality gap results for a single policy evaluation step are reported in Appendix \ref{sec:ADPPolicyEvaluation}.

\subsubsection{Experimental Results for  $m=3$}
The left panel of Figure \ref{figure:ADPOptPolicyComparison} shows the optimal risk-aware policies for $m=3$ under $\tau\in\{0.1,0.5,0.9\}$, representing risk-seeking, risk-neutral, and risk-averse policies, respectively, and three cost settings: the baseline $(10,1,20,5)$, high ordering cost $(50,1,20,5)$, and high wastage cost $(10,1,20,50)$. Detailed optimality-gap results for different policies are reported in Appendix \ref{sec:FurtherDetailsADP}. We observe that the optimal policy is no longer of the classical $(s,S)$ form. As $\tau$ increases, the ordering region expands to higher inventory levels to reduce shortage risk. High ordering costs shrink the ordering region but increase order sizes, whereas high wastage costs also shrink the ordering region while reducing order quantities.

\begin{figure}[htbp]
    \centering
    \includegraphics[width=0.45\linewidth]{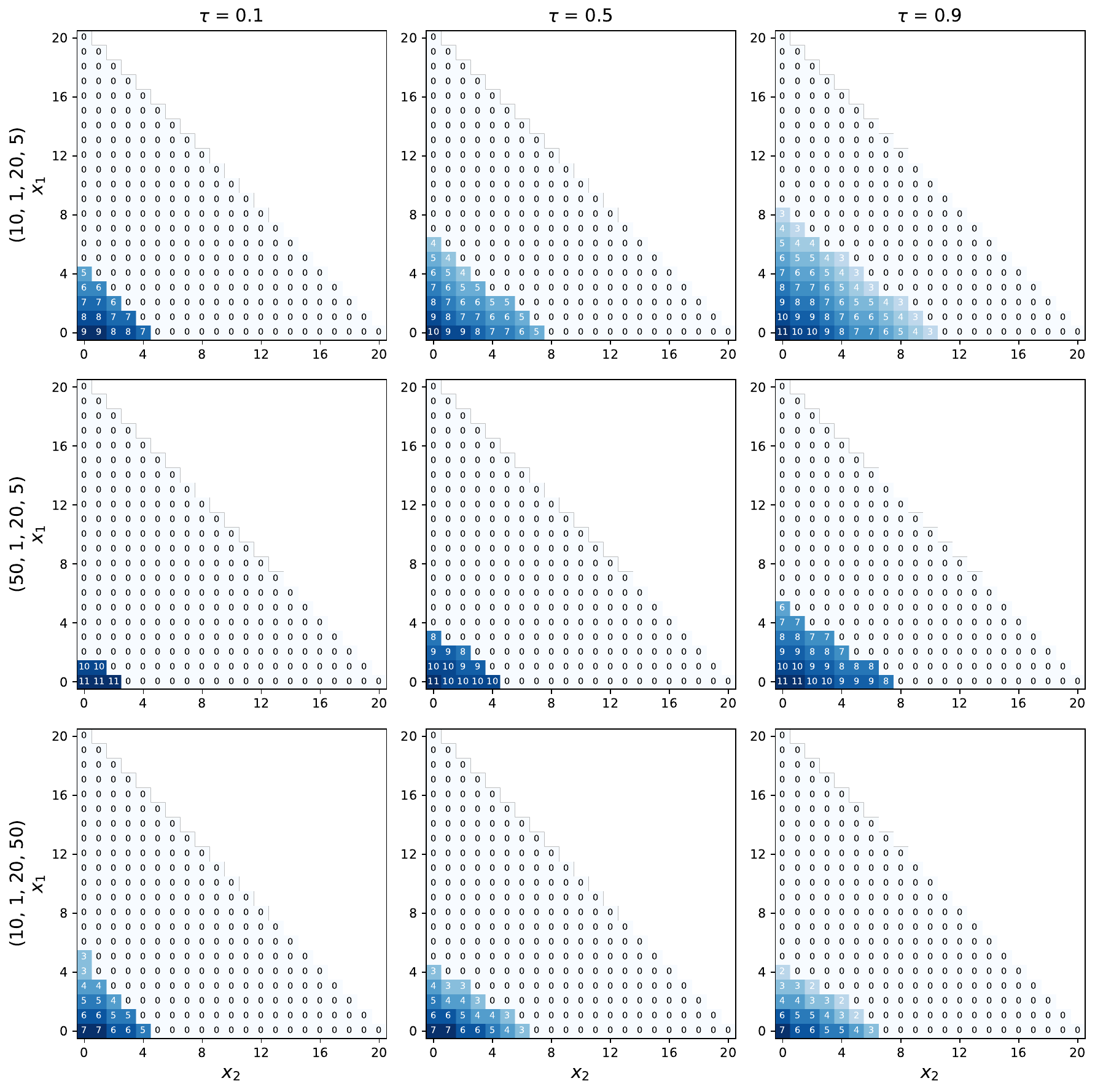}
    \hfill
    \includegraphics[width=0.45\linewidth]{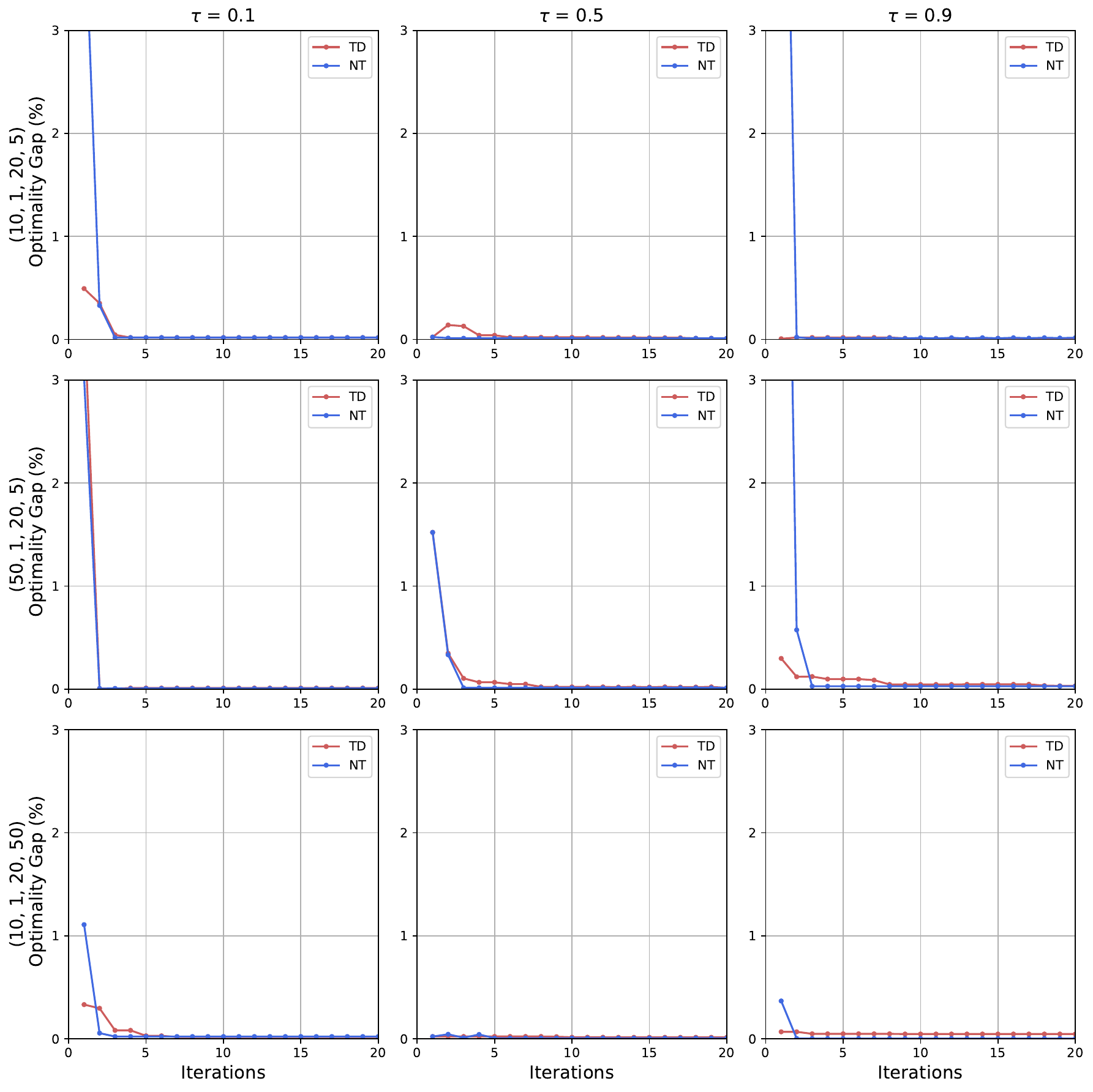}
    \caption{(Left) Optimal risk-aware policies for $m=3$; (Right) Optimality gaps of policy iteration with UBSR-TD and UBSR-Newton (TD and NT) for $m=3$.}
    \label{figure:ADPOptPolicyComparison}
\end{figure}

\begin{figure}[htbp]
    \centering
    \includegraphics[width=0.95\linewidth]{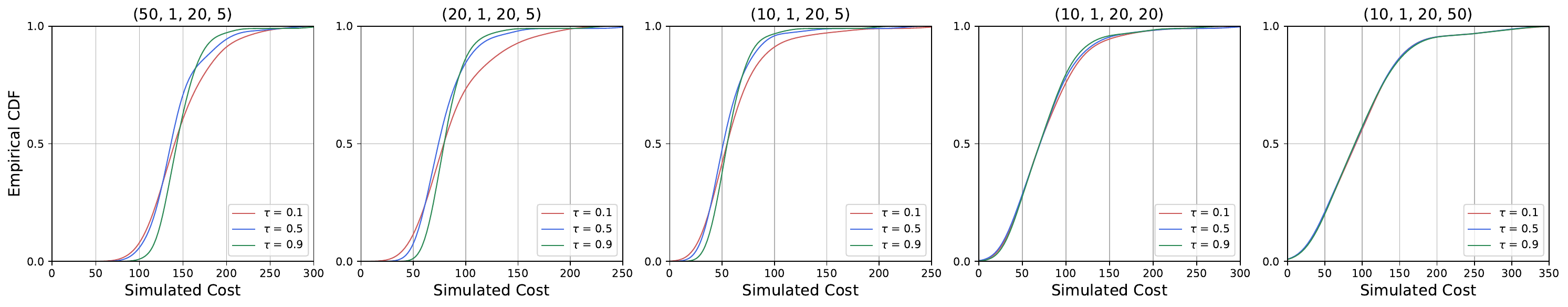}
    \caption{Empirical cumulative distributions of simulated costs under policies obtained via policy iteration with UBSR-TD under different risk and cost settings when $m=3$}
    \label{figure:ADPSimCosts}
\end{figure}

The right panel of Figure~\ref{figure:ADPOptPolicyComparison} reports the percentage optimality gap at the zero-inventory state over policy iterations for $m=3$, relative to the true optimum computed by value iteration with tolerance $0.01$. Both methods converge within about 10 iterations, with UBSR-Newton (blue) converging faster than UBSR-TD (red). After 20 iterations, both achieve gaps below $0.2\%$, demonstrating the effectiveness of the ADP framework and the proposed policy iteration methods. 

Figure~\ref{figure:ADPSimCosts} shows the empirical cumulative distributions of simulated costs under UBSR-TD policies for $m=3$ across different risk and cost settings. The distributions reflect the corresponding risk attitudes: the risk-seeking policy ($\tau=0.1$; red) yields greater dispersion and more extreme costs, the risk-averse policy ($\tau=0.9$; green) produces a more concentrated distribution with fewer extremes, and the risk-neutral policy ($\tau=0.5$; blue) lies between the two and minimizes the expectation.

Moreover, we observe that higher wastage costs diminish the influence of risk awareness on the optimal policy. When $c_4=50$, the risk-aware policies for different $\tau$ values become similar in Figure~\ref{figure:ADPOptPolicyComparison},
producing nearly identical cost distributions in Figure~\ref{figure:ADPSimCosts}. This occurs because when wastage costs are dominating, the optimal policy constrains the inventory to a narrow range around $7$, close to the expected demand of $6.2$.

\begin{table}[htbp]
  \centering
  \caption{Results under varying policies, risk parameters, and cost settings for $m=5$}
  \resizebox{\textwidth}{!}{
    \begin{tabular}{cccccccccccc}
    \toprule
    Cost  & $\tau$   & \multicolumn{2}{c}{0.1} & \multicolumn{2}{c}{0.4} & \multicolumn{2}{c}{0.5} & \multicolumn{2}{c}{0.6} & \multicolumn{2}{c}{0.9} \\
    \midrule
    \multirow{6}[2]{*}{(10,1,20,5)} & \textbf{Optimal} & \multicolumn{2}{c}{\textbf{31.23 }} & \multicolumn{2}{c}{\textbf{38.44 }} & \multicolumn{2}{c}{\textbf{40.23 }} & \multicolumn{2}{c}{\textbf{42.38 }} & \multicolumn{2}{c}{\textbf{52.98 }} \\
          & Static & 31.30  & (0.24\%) & 38.53  & (0.24\%) & 40.33  & (0.25\%) & 42.43  & (0.13\%) & 55.50  & (4.75\%) \\
          & Myopic & 31.86  & (2.02\%) & 40.03  & (4.15\%) & 41.23  & (2.47\%) & 44.01  & (3.84\%) & 54.19  & (2.27\%) \\
          & RN    & 33.18  & (6.27\%) & 38.60  & (0.43\%) &   -    &    -   & 42.55  & (0.39\%) & 66.04  & (24.65\%) \\
          & TD    & 31.30  & (0.24\%) & 38.53  & (0.24\%) & 40.43  & (0.49\%) & 42.40  & (0.05\%) & 54.13  & (2.16\%) \\
          & Newton & 31.30  & (0.24\%) & 38.53  & (0.24\%) & 40.43  & (0.49\%) & 42.40  & (0.05\%) & 53.79  & (1.54\%) \\
    \midrule
    \multirow{6}[2]{*}{(80,1,20,5)} & \textbf{Optimal} & \multicolumn{2}{c}{\textbf{120.40 }} & \multicolumn{2}{c}{\textbf{132.51 }} & \multicolumn{2}{c}{\textbf{136.13 }} & \multicolumn{2}{c}{\textbf{140.21 }} & \multicolumn{2}{c}{\textbf{165.96 }} \\
          & Static & 120.50  & (0.09\%) & 132.60  & (0.07\%) & 136.14  & (0.01\%) & 140.22  & (0.01\%) & 166.05  & (0.05\%) \\
          & Myopic & 188.46  & (56.53\%) & 162.65  & (22.75\%) & 165.49  & (21.57\%) & 173.77  & (23.93\%) & 200.82  & (21.00\%) \\
          & RN    & 121.66  & (1.05\%) & 132.60  & (0.07\%) &    -   &    -   & 140.37  & (0.11\%) & 170.86  & (2.95\%) \\
          & TD    & 120.70  & (0.25\%) & 132.60  & (0.07\%) & 136.14  & (0.01\%) & 140.23  & (0.01\%) & 166.04  & (0.05\%) \\
          & Newton & 120.41  & (0.01\%) & 132.60  & (0.07\%) & 136.14  & (0.01\%) & 140.22  & (0.01\%) & 166.04  & (0.05\%) \\
    \midrule
    \multirow{6}[2]{*}{(10,1,20,80)} & \textbf{Optimal} & \multicolumn{2}{c}{\textbf{31.38 }} & \multicolumn{2}{c}{\textbf{39.64 }} & \multicolumn{2}{c}{\textbf{42.55 }} & \multicolumn{2}{c}{\textbf{46.23 }} & \multicolumn{2}{c}{\textbf{84.09 }} \\
          & Static & 31.46  & (0.24\%) & 40.07  & (1.08\%) & 43.30  & (1.76\%) & 50.46  & (9.15\%) & 153.24  & (82.23\%) \\
          & Myopic & 31.95  & (1.81\%) & 40.69  & (2.63\%) & 42.82  & (0.63\%) & 46.51  & (0.60\%) & 101.66  & (20.90\%) \\
          & RN    & 33.17  & (5.70\%) & 39.88  & (0.60\%) &   -    &   -    & 46.33  & (0.22\%) & 87.66  & (4.24\%) \\
          & TD    & 31.46  & (0.24\%) & 39.76  & (0.29\%) & 42.64  & (0.21\%) & 46.49  & (0.55\%) & 84.56  & (0.56\%) \\
          & Newton & 31.46  & (0.24\%) & 39.74  & (0.25\%) & 42.63  & (0.17\%) & 46.33  & (0.21\%) & 84.14  & (0.06\%) \\
    \bottomrule
    \end{tabular}%
    }
  \label{table:ADPPolicyEvaluationm5}%
\end{table}%

\subsubsection{Experimental Results for $m=5$}
As discussed in Section \ref{sec:ApproximatePolicyIteration}, $m=5$ corresponds to a large-scale setting: exact value iteration takes about five hours to reach a precision of $0.1$ across the test instances, whereas the proposed policy iterations finish within one hour.

Table~\ref{table:ADPPolicyEvaluationm5} reports exact evaluations of policies obtained by policy iterations across different risk and cost settings, with static, myopic, and risk-neutral (RN) policies included as benchmarks. Further results on cost instances and discussions are reported in Appendix \ref{sec:FurtherDetailsADP}. Optimality gaps are shown in parentheses.
We observe that, the policy iteration methods achieve satisfactory optimality gaps across all scenarios: at worst $2.16\%$ for UBSR-TD and $1.54\%$ for UBSR-Newton, compared with $82.23\%$, $56.53\%$, and $24.65\%$ for the static, myopic, and risk-neutral benchmarks. For $\tau=0.1$ and $\tau=0.9$, the risk-aware policies consistently outperform the risk-neutral policy, with maximum improvements exceeding $5\%$ and $18\%$, highlighting the value of risk-aware control.

Overall, these results demonstrate that the proposed ADP framework with policy iteration using UBSR-TD and UBSR-Newton achieves robust performance across a wide range of scenarios, substantially outperforming static, myopic, and risk-neutral approaches.

\section{Conclusions}\label{sec:Conclusion}
In this work, we propose the first online policy evaluation algorithm for MDPs with dynamic UBSR measures under linear function approximation, the UBSR-TD algorithm, prove its almost sure convergence, and introduce several variants designed to accelerate convergence. 
Numerical experiments support our findings, and an application to perishable inventory management under shelf-life uncertainty highlights the effectiveness and practical value of the proposed methods.

Several directions worth further investigation. First, extending the framework to nonlinear function approximation, such as neural networks, is a natural next step. Second, relaxing the current theoretical restrictions on the discount factor remains an important open problem, as the existing bound is likely conservative. Finally, developing theoretical analyses for the proposed algorithmic extensions is another promising direction for future research.

\section*{Acknowledgments}
Weikai Wang is partially supported by IVADO and the Fonds de recherche du Québec [2007773]. Erick Delage is partially supported by the Natural Sciences and Engineering Research Council of Canada [Grant RGPIN-2022-05261].

\bibliography{reference}

\newpage


\appendix
\section*{Appendices}

\section{Algorithm Pseudo-codes}\label{sec:AlgorithmPseduoCodes}

\begin{algorithm}[H]
    \caption{UBSR-TD Algorithm}
    \label{algo-UBSR-TD}
    \KwIn{Online streaming data $\{X_n\}_{n\ge 0}$, initial parameter $\bmtheta_0$, $\gamma \in (0,1)$, step size $\{\eta_n\}_{n\ge 0}$, training horizon $T$}
    \KwOut{Finial estimator $\bmtheta_T$}

    \For{$n = 0$ \KwTo $T-1$}{
        Compute $\bmtheta_{n+1}$ through \eqref{eq-UBSR-TD}.
    }
\end{algorithm}

\begin{algorithm}[H]
    \caption{UBSR-TD($\lambda$) Algorithm}
    \label{algo-UBSR-TDlambda}
    \KwIn{Online streaming data $\{X_n\}_{n\ge 0}$, initial parameter $\bmtheta_0$, $\lambda \in [0,1]$, $\gamma \in (0,1)$, step size $\{\eta_n\}_{n\ge 0}$, training horizon $T$}
    \KwOut{Finial estimator $\bmtheta_T$}

    Set $\bmzeta_0 = \bmphi(X_0)$.
    
    \For{$n = 0$ \KwTo $T-1$}{
        
        Compute $\bmtheta_{n+1}$ through \eqref{eq-UBSR-TDLambda}.

        Compute $\bmzeta_{n+1}$ through \eqref{eq-UBSR-TDLambda-EligibilityTrace}
    }
\end{algorithm}

\begin{algorithm}[H]
    \caption{UBSR-Newton Algorithm}
    \label{algo-UBSR-Newton}
    \KwIn{Online streaming data $\{(\bmX_n,\bmZ_n,c_n)\}_{n \ge 1}$, initialization step $n_0 \ge d$, initial parameter $\hat{\bmtheta}_0$, $\gamma \in (0,1)$, training horizon $T$}
    \KwOut{Finial estimator $\hat{\bmtheta}_T$}
    
    Compute $\hat{\bmH}_{n_0}^{-1}$ and $\bmL_{n_0}$ according to \eqref{eq-PretrainHandL}.
    
    Set $\bar{\bmtheta}_{n_0} = \hat{\bmtheta}_{n_0} = \hat{\bmtheta}_0$.
    
    \For{ $n = n_0 + 1$ \KwTo $T$}{
        Compute $\bar{\bmtheta}_n$ through \eqref{eq-barthetaIteration}.
        
        Compute $\bmL_n$ and $\hat{\bmH}_n^{-1}$ through \eqref{eq-LIteration} and \eqref{eq-UBSRNewton-ShermanMorrison}.
        
        Update the parameter estimate $\hat{\bmtheta}_n$ through \eqref{eq-hattheta=bartheta-hatHinverseL}.
    } 
\end{algorithm}

\section{Omitted Proofs in Section \ref{sec:ConvergenceUBSRTD}}
\begin{proof}[Proof of Lemma \ref{lemma:mcHProperties}]
    We first show that $\mcH$ is a contraction on $(\mcL(\mcX),\norm{\cdot}_q)$.
    
    By Assumption \ref{assump-BoundedSlope} and the monotonicity of $\ell$, for any $a,b \in \mbR$, there exists some $\xi_{(a,b)}\in[\epsilon_1,L_1]$ such that $\ell(a) - \ell(b) = \xi_{(a,b)}(a-b)$.
    Using this property, for any $V_1, V_2 \in \mcL(\mcX)$,
    there exists some $\xi_{(x,x')} := \xi_{(x,x',V_1,V_2)} \in [\epsilon_1,L_1]$ such that
    \begin{equation*}
    \begin{aligned}
        &\norm{\mcH (V_1) - \mcH (V_2)}_q^2 \\
        =& \sum_{x\in\mcX} q(x) \Bigg( \alpha \sum_{x'\in\mcX} P(x'|x) \Big( \ell(\gamma V_1(x') - V_1(x) + c(x))- \ell(\gamma V_2(x') - V_2(x) + c(x)) \Big)  \\
        &\qquad + V_1(x) - V_2(x) \Bigg)^2\\
        =&\sum_{x\in\mcX} q(x) \Bigg( \alpha \sum_{x'\in\mcX} P(x'|x) \xi_{(x,x')} (\gamma \Delta(x') - \Delta(x)) + \Delta(x)\Bigg)^2\\
        =& \alpha^2 \sum_{x\in\mcX} q(x) \Bigg(\sum_{x'\in\mcX} P(x'|x) \xi_{(x,x')} (\gamma \Delta(x') - \Delta(x))\Bigg)^2\\
        & + 2\alpha \sum_{x\in\mcX} q(x) \sum_{x'\in\mcX} P(x'|x) \xi_{(x,x')} (\gamma \Delta(x') - \Delta(x))\Delta(x)
        + \sum_{x\in\mcX} q(x) (\Delta(x))^2\\
        =:& \alpha^2 \sum_{x\in\mcX} q(x) T_1(x) + 2\alpha \sum_{x\in\mcX} q(x) T_2(x) + \norm{\Delta}_q^2,
    \end{aligned}
    \end{equation*}
    where we define $\Delta(x) := V_1(x) - V_2(x)$, for all $x\in\mcX$. 
    For the first term, we have
    \begin{equation*}
    \begin{aligned}
        \sum_{x\in\mcX} q(x) T_1(x) &\le \sum_{x\in\mcX} q(x) \left( \sum_{x'\in\mcX} P(x'|x) \xi_{(x,x')} |\gamma \Delta(x') - \Delta(x)| \right)^2\\
        &\le \sum_{x\in\mcX} q(x) \sum_{x'\in\mcX} P(x'|x) \left( \xi_{(x,x')} |\gamma \Delta(x') - \Delta(x)| \right)^2\\
        &\le L_1^2 \sum_{x\in\mcX} q(x) \sum_{x'\in\mcX} P(x'|x) (\gamma \Delta(x') - \Delta(x))^2\\
        &= L_1^2 \sum_{x\in\mcX} q(x) \sum_{x'\in\mcX} P(x'|x) (\gamma^2 \Delta^2(x') + \Delta^2(x) - 2\gamma\Delta(x')\Delta(x))\\
        &\le L_1^2 (1+\gamma)^2 \norm{\Delta}_q^2,
    \end{aligned}
    \end{equation*}
    where the second inequality follows from Jensen's inequality; the third inequality follows from Assumption \ref{assump-BoundedSlope} and the last inequality from the invariant distribution equation and the Cauchy-Schwarz inequality:
    \begin{equation}\label{eq-CauchySchwarz}
        \left| \sum_{x\in\mcX} q(x) \sum_{x'\in\mcX}P(x'|x) \Delta(x')\Delta(x) \right| = |\mbE[\Delta(X')\Delta(X)]|\le (\mbE[\Delta^2(X')]\mbE[\Delta^2(X)])^{\frac{1}{2}} = \norm{\Delta}_q^2,
    \end{equation}
    where $X\sim q$ and $X'\sim P(\cdot|X)$. This inequality extends directly to the $n$-step case by replacing $x'$ with $x_n$ and $P$ with the $n$-step transition kernel $P^n$, yielding $|\mbE[\Delta(X_n)\Delta(X)]| \le \norm{\Delta}_q^2$, since $(P^n)^\top q = q$ by the invariant distribution equation.

    For the second term, we have
    \begin{equation*}
    \begin{aligned}
        \sum_{x\in\mcX} q(x) T_2(x) &=  \sum_{x\in\mcX} q(x) \sum_{x'\in\mcX} P(x'|x) \xi_{(x,x')} (\gamma \Delta(x')\Delta(x) - \Delta^2(x))\\
        &\le \sum_{x\in\mcX} q(x) \sum_{x'\in\mcX} P(x'|x) \xi_{(x,x')}\Big( \frac{\gamma}{2} (\Delta^2(x') + \Delta^2(x)) - \Delta^2(x)\Big)
        \le  ( \gamma L_1 - \epsilon_1) \norm{\Delta}_q^2,
    \end{aligned}
    \end{equation*}
    where the first inequality follows from the AM-GM inequality $ab \le \frac{1}{2}(a^2 + b^2)$, for all $a,b \in \mbR$, and the last inequality follows from the invariant distribution equation and Assumption \ref{assump-BoundedSlope}.

    As a result, we obtain that
    \begin{equation*}
        \norm{\mcH (V_1) - \mcH (V_2)}_q^2 \le \left( \alpha^2 L_1^2(1+\gamma)^2 + 2\alpha (\gamma L_1 - \epsilon_1) + 1\right) \norm{\Delta}_q^2.
    \end{equation*}
    To ensure that $\mcH$ is a contraction under $\norm{\cdot}_q$, we require
    that $g(\alpha) := \alpha^2 L_1^2(1+\gamma)^2 + 2\alpha (\gamma L_1 - \epsilon_1) + 1 < 1$. This gives $\alpha < \frac{2(\epsilon_1 - \gamma L_1)}{L_1^2 (1+\gamma)^2}$. Since $\alpha > 0$, we require that $\epsilon_1 - \gamma L_1 > 0$. 
    Moreover, we have $\min_{\alpha > 0} g(\alpha) = g(\frac{\epsilon_1-\gamma L_1}{L_1^2(1+\gamma)^2})=-\frac{(\epsilon_1-\gamma L_1)^2}{L_1^2(1+\gamma)^2}+1>0$. Therefore $g(\alpha) \in (0,1)$ for $\alpha \in (0, \frac{2(\epsilon_1 - \gamma L_1)}{L_1^2 (1+\gamma)^2})$.
    By Banach's fixed point theorem, there exists a unique fixed point $V$ for $\mcH$ on $(\mcL(\mcX),\norm{\cdot}_q)$ such that $V = \mcH V$. This implies that $V$ is also a solution to \eqref{eq-UBSRRoot}. 
\end{proof}

\begin{proof}[Proof of Lemma \ref{lemma:Phitheta-VstarBound}]
    By the Pythagorean theorem, for any $V \in \mcL(\mcX)$, we have
    \[ \norm{V}_q^2 = \norm{\Pi V}_q^2 + \norm{V - \Pi V}_q^2, \]
    since $\langle \Pi V, V - \Pi V\rangle_q = 0$. By Theorem \ref{lemma:mcHProperties}, $\mcH$ is a contraction under $\norm{\cdot}_q$ and $\Pi$ is non-expansive, it follows that $\Pi \mcH$ is a contraction on $(\mcL(\mcX), \norm{\cdot}_q)$. Hence, by Banach's fixed point theorem, $\Pi \mcH$ admits a unique fixed point. Moreover, by the definition of the projection operator and $\{\bmphi_i\}_{i=1}^d$ are linearly independent, this fixed point can be expressed as $\bmPhi \bmtheta^*$ for some unique vector $\bmtheta^* \in \mbR^d$.

    Since $V \in \mcL(\mcX)$ and is the fixed point of $\mcH$, we have
    \begin{equation*}
        \norm{\bmPhi \bmtheta^* - \Pi V}_q = \norm{\Pi\mcH(\bmPhi\bmtheta^*) - \Pi V}_q
        \le \norm{\mcH(\bmPhi\bmtheta^*) - V}_q
        = \norm{\mcH(\bmPhi\bmtheta^*) - \mcH V}_q
        \le \bar{\alpha} \norm{\bmPhi\bmtheta^* - V}_q,
    \end{equation*}
    where $\bar{\alpha}$ is the contraction factor for $\mcH$. Note that
    \begin{equation*}
        \norm{\bmPhi\bmtheta^* - V}_q \le \norm{\bmPhi \bmtheta^* - \Pi V}_q + \norm{\Pi V - V}_q
        \le \bar{\alpha} \norm{\bmPhi\bmtheta^* - V}_q + \norm{\Pi V - V}_q.
    \end{equation*}
    It follows that $\norm{\bmPhi\bmtheta^* - V}_q \le \frac{1}{1-\bar{\alpha}}\norm{\Pi V - V}_q$.
    This completes the proof.
\end{proof}

\begin{proof}[Proof of Lemma \ref{lemma:ODEUniqueEquilibrium}]
    From Lemma \ref{lemma:Phitheta-VstarBound}, we know that there exists a unique $\bmtheta^* \in \mbR^d$ such that $\bmPhi\bmtheta^* = \Pi\mcH (\bmPhi \bmtheta^*)$, i.e., $\bmPhi\bmtheta^*$ is the unique fixed point for $\Pi\mcH$. From the Hilbert projection theorem (Corollary 5.4, \cite{Brezis2010book}), we know that $\langle \mcH (\bmPhi \bmtheta^*) - \bmPhi \bmtheta^*, \bmu \rangle_q = 0$ for any $\bmu \in \mathrm{span}(\bmPhi)$. Since $q > 0$, such $\bmPhi \bmtheta^*$ is unique in $\mathrm{span}(\bmPhi)$. Since $\bmPhi$ is full rank, $\bmtheta^*$ is also unique.
    This implies that $\bmtheta^*$ is the unique solution to
    \begin{equation}\label{eq-FixedPointEquationPiHMatrix}
        \bmPhi^\top \diag(q) (\mcH (\bmPhi \bmtheta^*) - \bmPhi \bmtheta^*) = \bm{0}.
    \end{equation}
    Recalling the definition of $\mcH$, \eqref{eq-FixedPointEquationPiHMatrix} can be equivalently written as
    \begin{equation*}
        \sum_{x\in\mcX} q(x) \bmphi(x) \sum_{x'\in\mcX} P(x'|x) \ell\left((\gamma \bmphi(x') - \bmphi(x))^\top \bmtheta^* + c(x)\right) = \bm{0}.
    \end{equation*}
    This implies that $\bmtheta^*$ is the unique solution of \eqref{eq-ODE-equilibrium}, i.e., the unique equilibrium of the ODE \eqref{eq-TD-ODE}. This completes the proof.  
\end{proof}

Before proving Theorem \ref{thm:UBSRTDASConvergence}, we present several lemmas that are necessary to establish the theorem. We first show that $\hat{\mcH}$ defined in \eqref{eq-OperatorhatH} is a contraction on $(\mcL(\mbR^d), \norm{\cdot}_2)$.

\begin{lemma}\label{lemma:hatHContraction}
    Under the assumptions of Lemma \ref{lemma:mcHProperties}, 
    the operator $\hat{\mcH}$ is a contraction on $(\mcL(\mbR^d),\norm{\cdot}_2)$.
\end{lemma}
\begin{proof}
    From \eqref{eq-FixedPointEquationPiHMatrix} and Assumption \ref{assump-BoundedSlope}, we know that for any $\bmtheta_1,\bmtheta_2 \in \mbR^d$ with $\bmtheta_1 \neq \bmtheta_2$, there exists some $\xi_{(x,x',\bmtheta_1,\bmtheta_2)}\in [\epsilon_1,L_1]$ such that
    \begin{equation*}
    \begin{aligned}
        &(\bmtheta_1 - \bmtheta_2)^\top \bmPhi^\top \diag(q) (\mcH(\bmPhi\bmtheta_1) - \bmPhi\bmtheta_1 - \mcH (\bmPhi\bmtheta_2) + \bmPhi\bmtheta_2)\\
        =& (\bmtheta_1 - \bmtheta_2)^\top \alpha \sum_{x\in\mcX} q(x) \bmphi(x) \sum_{x'\in\mcX} P(x'|x) \Big\{ \ell\left((\gamma \bmphi(x') - \bmphi(x))^\top \bmtheta_1 + c(x)\right) \\
        &\qquad - \ell\left((\gamma\bmphi(x') - \bmphi(x))^\top \bmtheta_2 + c(x)\right) \Big\}\\
        = &(\bmtheta_1 - \bmtheta_2)^\top \alpha \sum_{x\in\mcX} q(x) \bmphi(x) \sum_{x'\in\mcX} P(x'|x) \xi_{(x,x',\bmtheta_1,\bmtheta_2)} (\gamma\bmphi(x') - \bmphi(x))^\top (\bmtheta_1 - \bmtheta_2)\\
        = &\alpha (\bmtheta_1 - \bmtheta_2)^\top \bmXi(\bmtheta_1,\bmtheta_2) (\bmtheta_1 - \bmtheta_2),
    \end{aligned}
    \end{equation*}
    where we denote 
    \begin{equation}\label{eq-bmXi}
        \bmXi(\bmtheta_1,\bmtheta_2) :=
        \sum_{x\in\mcX} q(x) \bmphi(x) \sum_{x'\in\mcX} P(x'|x) \xi_{(x,x',\bmtheta_1,\bmtheta_2)} (\gamma \bmphi(x') - \bmphi(x))^\top \in \mbR^{d\times d}.
    \end{equation}
    Putting $(\bmtheta_1-\bmtheta_2)^\top$ inside the summations, we have
    \begin{equation}\label{eq-lemma:hatHContraction-1}
    \begin{aligned}
        &\alpha (\bmtheta_1 - \bmtheta_2)^\top \bmXi(\bmtheta_1,\bmtheta_2) (\bmtheta_1 - \bmtheta_2)\\
        =&\alpha \sum_{x\in\mcX} \sum_{x'\in \mcX} q(x) P(x'|x) (\bmtheta_1 - \bmtheta_2)^\top \bmphi(x)\xi_{(x,x',\bmtheta_1,\bmtheta_2)} (\gamma \bmphi(x') - \bmphi(x))^\top(\bmtheta_1 - \bmtheta_2)\\
        =& \alpha \sum_{x\in\mcX}\sum_{x'\in \mcX} q(x) P(x'|x) \xi_{(x,x',\bmtheta_1,\bmtheta_2)} (\gamma f_{\bmdelta}(x') f_{\bmdelta}(x) - f_{\bmdelta}^2(x))\\
        \le & \alpha \sum_{x\in\mcX} \sum_{x'\in\mcX} q(x) P(x'|x) \xi_{(x,x',\bmtheta_1,\bmtheta_2)}\left( \gamma \frac{f_{\bmdelta}^2(x') + f_{\bmdelta}^2(x)}{2} - f_{\bmdelta}^2(x) \right)\\
        =& \alpha \sum_{x\in\mcX} \sum_{x'\in \mcX} q(x) P(x'|x) \xi_{(x,x',\bmtheta_1,\bmtheta_2)} \left( \frac{\gamma}{2} - 1 \right) f_{\bmdelta}^2(x) \\
        &\qquad + \alpha \sum_{x\in\mcX} \sum_{x'\in \mcX} q(x) P(x'|x) \xi_{(x,x',\bmtheta_1,\bmtheta_2)} \frac{\gamma}{2} f_{\bmdelta}^2(x')\\
        \le & \alpha \left( \frac{\gamma}{2} - 1 \right) \epsilon_1 \sum_{x\in\mcX}\sum_{x'\in \mcX} q(x) P(x'|x) f_{\bmdelta}^2(x) + \frac{\alpha \gamma L_1}{2} \sum_{x\in\mcX}\sum_{x'\in\mcX} q(x) P(x'|x) f_{\bmdelta}^2(x')\\
        =& \alpha \left( \frac{\gamma(\epsilon_1 + L_1)}{2} - \epsilon_1 \right) \norm{f_{\bmdelta}}_q^2 < 0,
    \end{aligned}
    \end{equation}
    where we denote $f_{\bmdelta}(x) := \bmphi^\top(x)\bmdelta \in \mbR$ and $f_{\bmdelta}:= \bmPhi\bmdelta \in \mbR^{N}$ with $\bmdelta := \bmtheta_1 - \bmtheta_2 \in \mbR^d$. The first inequality follows from the AM-GM inequality, the second inequality follows from the invariant distribution equality $P^\top q = q$ and the exchange of sum over $x$ and $x'$, the last inequality follows from the fact that that we require $\gamma < \epsilon_1/L_1$, hence $\frac{\gamma(\epsilon_1 + L_1)}{2} - \epsilon_1 < 0$. 

    Before moving forward, we introduce 
    an assumption on $\bmPhi$, which will be seen shortly to be non-restrictive. Let $\lambda_M, \lambda_m$ denote the largest and smallest singular value of $\sqrt{\diag(q)}\bmPhi$, and suppose that
    \begin{equation*}
        0 < \lambda_m \le \lambda_M < \sqrt{\frac{2\epsilon_1 -\gamma(\epsilon_1 + L_1)}{ \alpha^2 L_1^2 (\gamma + 1)^2}},
    \end{equation*}
    where $\alpha>0$ is a constant satisfying the conditions of Lemma \ref{lemma:mcHProperties}. Here we have $\lambda_m > 0$ since we have assumed that $\bmPhi$ has full column rank and $q > 0$. Because the feature vectors may be rescaled without altering the algorithm (with the parameter vector $\bmtheta$ get scaled accordingly), this assumption entails no loss of generality (also see Section 2.1 in \cite{Chandak2025stsy}). 
    
    On the other hand, we have
    \begin{equation*}
    \begin{aligned}
        & \norm{\alpha \bmXi(\bmtheta_1,\bmtheta_2)(\bmtheta_1 - \bmtheta_2)}_2^2\\
        =& \norm{\bmPhi^\top \diag(q) (\mcH(\bmPhi\bmtheta_1) - \bmPhi\bmtheta_1 - \mcH(\bmPhi\bmtheta_2) + \bmPhi\bmtheta_2)}_2^2\\
        =& \norm{ \bmPhi^\top \sqrt{\diag(q)}\sqrt{\diag(q)} (\mcH(\bmPhi\bmtheta_1) - \bmPhi\bmtheta_1 - \mcH(\bmPhi\bmtheta_2) + \bmPhi\bmtheta_2)}_2^2\\
        \le & \lambda_M^2 \alpha^2  \sum_{x\in\mcX} \Bigg( \sqrt{q(x)} \sum_{x'\in\mcX} P(x'|x) \Big( \ell\big( (\gamma\bmphi(x') - \bmphi(x))^\top \bmtheta_1 + c(x) \big)  \\
        &\qquad - \ell\big( (\gamma \bmphi(x') - \bmphi(x))^\top \bmtheta_2 + c(x) \big) \Big) \Bigg)^2\\
        \le & \lambda_M^2 \alpha^2  \sum_{x\in\mcX}   q(x) \sum_{x'\in\mcX} P(x'|x) \Bigg( \xi_{(x,x',\bmtheta_1,\bmtheta_2)} (\gamma\bmphi(x') - \bmphi(x))^\top (\bmtheta_1 - \bmtheta_2) \Bigg)^2\\
        \le & \lambda_M^2 \alpha^2 L_1^2 \sum_{x\in\mcX} q(x) \sum_{x'\in\mcX} P(x'|x) (\gamma f_{\bmdelta}(x') - f_{\bmdelta}(x))^2\\
        \le & \lambda_M^2 \alpha^2 L_1^2 (\gamma +1)^2 \norm{f_{\bmdelta}}_q^2,
    \end{aligned}
    \end{equation*}
    where the first inequality follows from the definition of the spectral norm induced by the Euclidean norm, the second inequality is obtained using Assumption \ref{assump-BoundedSlope} together with Jensen's inequality, and the final inequality is a consequence of the Cauchy-Schwarz inequality \eqref{eq-CauchySchwarz}.
    
    Therefore, for any $\bmtheta_1,\bmtheta_2 \in \mbR^d$ with $\bmtheta_1 \neq \bmtheta_2$, we have
    \begin{equation*}
    \begin{aligned}
        &\hat{\mcH} (\bmtheta_1) - \hat{\mcH} (\bmtheta_2) \\
        =& \alpha \sum_{x\in\mcX} q(x) \bmphi(x) \sum_{x'\in\mcX} P(x'|x) \Bigg( \ell\left( \gamma \bmphi(x') - \bmphi(x))^\top \bmtheta_1 + c(x) \right)  \\
        &\qquad -  \ell\left( \gamma \bmphi(x') - \bmphi(x))^\top \bmtheta_2 + c(x) \right) \Bigg) + \bmtheta_1 - \bmtheta_2\\
        =& \alpha \sum_{x\in\mcX} q(x) \bmphi(x) \sum_{x'\in\mcX} P(x'|x) \xi_{(x,x',\bmtheta_1,\bmtheta_2)} (\gamma \bmphi(x') - \bmphi(x))^\top (\bmtheta_1 - \bmtheta_2) + \bmtheta_1 - \bmtheta_2\\
        =& \left(\bmI + \alpha \bmXi(\bmtheta_1,\bmtheta_2)\right)(\bmtheta_1 - \bmtheta_2).
    \end{aligned}
    \end{equation*}
    Taking the Euclidean norm, we obtain that
    \begin{equation*}
    \begin{aligned}
        & \norm{\hat{\mcH}(\bmtheta_1) - \hat{\mcH}(\bmtheta_2)}_2^2 \\
        \le & \norm{\bmtheta_1 - \bmtheta_2}_2^2 + \alpha (\bmtheta_1 - \bmtheta_2)^\top \left( \bmXi(\bmtheta_1,\bmtheta_2) + \bmXi^\top(\bmtheta_1,\bmtheta_2)\right) (\bmtheta_1 - \bmtheta_2) \\
        &\qquad + \alpha^2 (\bmtheta_1 - \bmtheta_2)^\top \bmXi^\top(\bmtheta_1,\bmtheta_2) \bmXi(\bmtheta_1,\bmtheta_2) (\bmtheta_1 - \bmtheta_2) \\
        \le& \norm{\bmtheta_1-\bmtheta_2}_2^2 + \alpha \left( \gamma(\epsilon_1 + L_1) - 2 \epsilon_1 \right) \norm{f_{\bmdelta}}_q^2 + \lambda_M^2 \alpha^2 L_1^2(\gamma + 1)^2 \norm{f_{\bmdelta}}_q^2\\
        \le & \tilde{\beta} \norm{\bmtheta_1 - \bmtheta_2}_2^2,
    \end{aligned}
    \end{equation*}
    where the last inequality follows from the fact that 
    \begin{equation*}
    \begin{aligned}
        \norm{f_{\bmdelta}}_q^2 & = \norm{\bmPhi(\bmtheta_1 - \bmtheta_2)}_q^2 = \norm{\sqrt{\diag(q)} \bmPhi (\bmtheta_1 - \bmtheta_2)}_2^2  \\
        & \ge \left( \min_{\bmtheta \neq \bm{0}}\frac{\norm{ \sqrt{\diag(q)} \bmPhi\bmtheta}_2^2}{\norm{\bmtheta}_2^2}\right) \norm{\bmtheta_1 - \bmtheta_2}_2^2 = \lambda_m^2 \norm{\bmtheta_1 - \bmtheta_2}_2^2,
    \end{aligned}
    \end{equation*}
    and the definition of $\lambda_M$ that
    \begin{equation*}
        \alpha \left( \gamma(\epsilon_1 + L_1) - 2 \epsilon_1 \right)  + \lambda_M^2 \alpha^2 L_1^2(\gamma + 1)^2 < 0,
    \end{equation*}
    which yields a contraction factor
    \begin{equation*}
        \tilde{\beta} := 1 - \lambda_m^2 \left(- \alpha (\gamma(\epsilon_1 + L_1) - 2\epsilon_1) - \lambda_M^2 \alpha^2 L_1^2 (\gamma + 1)^2\right) < 1.
    \end{equation*}
    Therefore we conclude that $\hat{\mcH}$ is a contraction on $(\mcL(\mbR^d),\norm{\cdot}_q)$. 
    This completes the proof.
\end{proof}

Next we prove some property regarding $H$ defined in \eqref{eq-UBSR-TD-SAwithMarkovNoise}. We first show a property of the loss function $\ell$.
\begin{lemma}[Lemma B.19 in \cite{WangWK2025neurips}]\label{lemma:ellinftyExistence}
    For any loss function $\ell: \mbR \to \mbR$ with $\ell(0) = 0$ satisfying Assumptions \ref{assump-BoundedSlope} and \ref{assump-Convexity}, define $\ell_s(z) := \frac{1}{s} \ell(sz)$, for all $z\in \mbR$. We have $\ell_\infty(z) := \lim_{s\to\infty} \ell_s(z)$ exists and the convergence is uniform on compact sets. Furthermore, $\ell_\infty$ also satisfies Assumptions \ref{assump-BoundedSlope} and \ref{assump-Convexity}.
\end{lemma}

Define $H_s(\bmtheta,y) := \frac{1}{s}H(s\bmtheta,y)$, for $s \ge 1$, as a rescaled version of the function $H$ with respect to $\bmtheta$. We also define $h_s(\bmtheta) := \mbE_{Y\sim d_Y} H_s(\bmtheta,Y)$, where $d_Y$ is the stationary distribution for the Markov chain $\{Y_n\}_{n\ge 1}$.

\begin{lemma}\label{lemma:HinftyhinftyExistence}
    Under the assumptions of Lemma \ref{lemma:mcHProperties}, $H_\infty(\bmtheta,y) := \lim_{s\to\infty} H_s(\bmtheta,y)$ exists and is finite for each $\bmtheta \in \mbR^d$ and $y \in \mcY$. Consequently, $h_\infty(\bmtheta) := \mbE_{Y \sim d_Y}[H_\infty(\bmtheta,Y)]$ exists and is finite, and $\lim_{s\to\infty} h_s(\bmtheta) = h_\infty(\bmtheta)$ for each $\bmtheta \in \mbR^d$.
\end{lemma}
\begin{proof}
    By definition, for any $\bmtheta \in \mbR^d$, $y =(x,x') \in \mcY$, we have
    \begin{equation*}
    \begin{aligned}
        \lim_{s\to\infty}H_s(\bmtheta,y) &= \lim_{s\to\infty} \frac{1}{s} \alpha \bmphi(x)\ell\left( (\gamma \bmphi(x') - \bmphi(x))^\top s \bmtheta + c(x) \right)\\
        &= \alpha \bmphi(x) \ell_\infty\left( (\gamma \bmphi(x') - \bmphi(x))^\top \bmtheta \right) =: H_\infty(\bmtheta,y),
    \end{aligned}
    \end{equation*}
    exists and is finite. Here the second equality follows from the fact that $\lim_{s\to\infty}|\frac{1}{s}\ell(sz + c) - \frac{1}{s}\ell(sz)| \le \lim_{s\to\infty} \frac{1}{s} L_1 |sz + c -sz| = 0$ for any $c \in \mbR$, hence following Lemma \ref{lemma:ellinftyExistence}, $\lim_{s\to\infty} \frac{1}{s}\ell(sz + c) = \ell_\infty(z)$. As a result, $h_\infty(\bmtheta) :=\mbE_{Y\sim q_{Y}}[H_\infty(\bmtheta,Y)]$ exists and is finite for any fixed $\bmtheta \in \mbR^d$ since 
    \begin{equation*}
    \begin{aligned}
        \norm{h_\infty(\bmtheta)}_2 &= \alpha \left\| \mbE[\bmphi(X)(\ell_\infty((\gamma \bmphi(X') - \bmphi(X))^\top\bmtheta)) ]  \right\|_2\\
        & \le \alpha \mbE\left[ \left\| \bmphi(X) \left(\ell_\infty((\gamma \bmphi(X') - \bmphi(X))^\top \bmtheta) - \ell_\infty(0)\right) \right\|_2 \right]
        \le \alpha M^2 L_1 (\gamma + 1) \norm{\bmtheta}_2 < \infty,
    \end{aligned}
    \end{equation*}
    for any random states $X,X'$ in $\mcX$,
    where the first inequality is from Jensen's inequality and the fact that $\ell_\infty(0) = 0$; the second inequality follows from the triangular inequality, Assumptions \ref{assump-BoundedSlope}, \ref{assump-phibounded}.
    
    Note that $\norm{H_s(\bmtheta,y)}_2 \le \alpha ML_1( (\gamma+1)M \norm{\bmtheta}_2 + \bar{c})$ for all $s\ge 1$, $\bmtheta \in \mbR^d, y\in\mcY$. By the bounded convergence theorem, we have $\lim_{s\to\infty} h_s(\bmtheta) = h_\infty(\bmtheta)$ for each $\bmtheta \in \mbR^d$.    
    This completes the proof.
\end{proof}

\begin{lemma}\label{lemma:HLipschitz}
    The function $H$ is Lipschitz continuous with respect to $\bmtheta$, i.e., there exists a constant $L_{2}$ such that $\norm{H(\bmtheta_1,y) - H(\bmtheta_2,y)}_2 \le L_{2}\norm{\bmtheta_1 - \bmtheta_2}_2$, for any $\bmtheta_1,\bmtheta_2 \in \mbR^d$, $y\in\mcY$.
\end{lemma}
\begin{proof}
    By Assumption \ref{assump-BoundedSlope}, for $y = (x,x') \in \mcY$, for any $\bmtheta_1, \bmtheta_2 \in \mbR^d$, there exists some $\xi_{(x,x',\bmtheta_1,\bmtheta_2)} \in [\epsilon_1,L_1]$ such that,
    \begin{equation*}
    \begin{aligned}
        \norm{H(\bmtheta_1,y) - H(\bmtheta_2,y)}_2 &= \Big\|\alpha \bmphi(x) \left( \xi_{(x,x',\bmtheta_1,\bmtheta_2)} (\gamma \bmphi(x') - \bmphi(x))^\top (\bmtheta_1 - \bmtheta_2) \right)\Big\|_2 \\
        &\le \alpha M^2 L_1 (\gamma + 1) \norm{\bmtheta_1 - \bmtheta_2}_2.
    \end{aligned}
    \end{equation*}
    This completes the proof.
\end{proof}

We invoke a recent stochastic approximation result with Markovian noise to establish the almost sure convergence of our UBSR-TD algorithm \eqref{eq-UBSR-TD}. We refer readers to Section 2.2 of \cite{Borkar2025aoap} for the definitions of a small function and aperiodicity of a Markov chain.

\begin{theorem}[Theorem 1, \cite{Borkar2025aoap}]\label{thm:BorkarSATheorem}
    Consider the stochastic approximation
    \begin{equation}\label{eq-SAMarkovian}
        \bmtheta_{n+1} = \bmtheta_n + \eta_n f(\bmtheta_n, Y_{n+1}),
    \end{equation}
    where $f:\mbR^d \times \mcY \to \mbR^d$, $\{Y_n\}_{n\ge 1}$ is an irreducible and aperiodic Markov chain on $\mcY$, with unique stationary distribution $d_Y$, and transition kernel $P$. Suppose the following conditions hold:
    \begin{enumerate}[label = (\roman*)]
        \item The step size $\{\eta_n\}_{n\ge 1}$ satisfies Assumption \ref{assump-StepSizeLimit}.
        
        \item There exists a constant $\mfL$ such that for each $y \in \mcY$, $\norm{f(\bm{0},y)} \le \mfL$, and 
        \[ \norm{ f(\bmtheta_1, y) - f(\bmtheta_2, y)} \le \mfL \norm{ \bmtheta_1 - \bmtheta_2}, \quad \forall \bmtheta_1,\bmtheta_2 \in \mbR^d.\] 

        \item The Lyapunov drift criterion (V4) holds with respect to the Lyapunov function $v: \mcY \to [1,\infty)$, i.e., for a small function $s$ and constants $\delta > 0$, $b < \infty$,
        \[ \mbE[v(Y_{n+1}) - v(Y_{n}) | Y_n = y] \le -\delta v(y) + b s(y),\quad \forall y \in \mcY. \]
        
        \item The scaled vector field $\bar{f}_\infty(\bmtheta)$ exists for each $\bmtheta \in \mbR^d$, where $\bar{f}_\infty(\bmtheta) := \lim_{r\to\infty} \bar{f}_r(\bmtheta)$, $\bar{f}_r(\bmtheta) := \mbE_{Y\sim d_Y}[f_r(\bmtheta,Y)]$, and $f_r(\bmtheta,y) := \frac{1}{r}f(r\bmtheta,y)$ for $r \ge 1$, $y \in \mcY$.  
        Moreover, the ODE 
        \begin{equation}\label{eq-BorkarODEbarfinfty}
            \frac{d \bmvartheta(t)}{dt} = \bar{f}_\infty(\bmvartheta(t)),
        \end{equation}
        is globally asymptotically stable.

        \item The ODE 
        \begin{equation}\label{eq-BorkarODEbarf}
            \frac{d\bmvartheta(t)}{dt} = \bar{f}(\bmvartheta(t)),
        \end{equation}
        is globally asymptotically stable, where $\bar{f}(\bmtheta) := \mbE_{Y\sim d_Y}[f(\bmtheta,Y)]$, for all $\bmtheta \in \mbR^d$.
    \end{enumerate}
    Then the sequence $\{\bmtheta_n\}_{n\ge 0}$ generated by \eqref{eq-SAMarkovian} converges a.s. to the invariant set of the ODE \eqref{eq-BorkarODEbarf}.
\end{theorem}
We note that in \cite{Borkar2025aoap}, the transition kernel is allowed to depend on $\bmtheta$ and the Lipschitz constant $\mfL$ in (ii) can be a function of $y$, which necessitates additional continuity and measurable assumptions (see inequality (30) and Assumption (A2) in \cite{Borkar2025aoap}). In contrast, since the transition kernel in our framework does not depend on $\bmtheta$ and $\mfL$ is a constant, these assumptions are omitted from Theorem~\ref{thm:BorkarSATheorem} for clarity.

\begin{proof}[Proof of Theorem \ref{thm:UBSRTDASConvergence}]
    It is straightforward to see that \eqref{eq-UBSR-TD-SAwithMarkovNoise} is the stochastic approximation with Markovian noise \eqref{eq-SAMarkovian} with $f(\bmtheta,y) := \frac{1}{\alpha} H(\theta,y)$, $\bar{f}(\bmtheta) := \frac{1}{\alpha} h(\bmtheta)$, $\bar{f}_\infty(\bmtheta) := \frac{1}{\alpha} h_\infty(\bmtheta)$, $\bar{f}_r(\bmtheta) := \frac{1}{\alpha} h_r(\bmtheta)$, with $H$, $h$, $h_\infty$, $h_r$ defined in \eqref{eq-UBSR-TD-SAwithMarkovNoise}, \eqref{eq-TD-ODE}, Lemma \ref{lemma:HinftyhinftyExistence}, respectively,  for all $\bmtheta \in \mbR^d$ and $y\in\mcY$.
    We show our algorithm satisfies all the conditions required by Theorem \ref{thm:BorkarSATheorem}.
    
    (i) Condition (i) is satisfied by enforcing Assumption \ref{assump-StepSizeLimit}.

    (ii) Condition (ii) is proved by Lemma \ref{lemma:HLipschitz} by letting $\mfL := L_2/\alpha$.

    (iii) Under Assumption~\ref{assump-StationaryDistribution}, the Markov chain $\{Y_n\}_{n \ge 1}$ is finite, irreducible and aperiodic. The Lyapunov drift condition (V4) holds trivially with $v(y) \equiv 1$, $s(y) \equiv 1$ for all $y \in \mcY$ and any $b > \delta$. In particular,
    \[
    \mbE\left[ v(Y_{n+1}) - v(Y_n) \mid Y_n = y \right]
    = 0 \le -\delta + b, \qquad \forall\, y \in \mcY.
    \]

    (v) Lemma \ref{lemma:hatHContraction} shows that $\hat{\mcH}$ is a contraction, while Lemma \ref{lemma:ODEUniqueEquilibrium} establishes that $\bmtheta^*$ is a unique equilibrium point of the ODE \eqref{eq-TD-ODE}. Consequently, we have $\bm{0} = \hat{\mcH}(\bmtheta^*) - \bmtheta^*$, implying that $\bmtheta^*$ is a fixed point of $\hat{\mcH}$. Then by Theorem 12.1 in \cite{Borkar2023book}, it follows that $\bmtheta^*$ is globally asymptotically stable for the ODE \eqref{eq-TD-ODE}. Since $\alpha > 0$ is a constant, it follows that $\bmtheta^*$ is also globally asymptotically stable for the ODE \eqref{eq-BorkarODEbarf}.

    (iv) Lemma~\ref{lemma:HinftyhinftyExistence} establishes that $h_\infty(\bmtheta)$ exists, is finite, and satisfies
    $h_\infty(\bmtheta) = \lim_{s\to\infty} h_s(\bmtheta)$ for all $\bmtheta \in \mbR^d$. Consequently, it follows that
    $\bar{f}_\infty(\bmtheta) = \lim_{s\to\infty} \bar{f}_s(\bmtheta)$ also exists and is finite for all $\bmtheta \in \mbR^d$.
    The ODE 
    \begin{equation}\label{eq-ODE-hinfty}
        \frac{d \bmvartheta(t)}{dt} = h_\infty(\bmvartheta(t)),
    \end{equation}
    is then the ODE \eqref{eq-TD-ODE} with $c(x) \equiv 0$, $\ell(x) = \ell_\infty(x)$ for all $x\in\mcX$, and $\bm{0}$ is an equilibrium for \eqref{eq-ODE-hinfty}. Then following a similar step as the proof of (v), we conclude that such an equilibrium point is globally asymptotically stable. Moreover, it is also a globally asymptotically stable equilibrium for the ODE \eqref{eq-BorkarODEbarfinfty}.

    Therefore, by applying Theorem \ref{thm:BorkarSATheorem}, we obtain that the sequence $\{\bmtheta_n\}_{n\ge 1}$ generated by the UBSR-TD algorithm \eqref{eq-UBSR-TD} converges almost surely to $\bmtheta^*$. The rest of the result then follows from Lemma \ref{lemma:Phitheta-VstarBound}. This completes the proof.
\end{proof}

\section{Further Experiments}\label{sec:FurtherExperiments}

\subsection{Convergence Experiments for Soft Quantile}\label{sec:ConvergenceSoftQuantile}

Figure~\ref{figure:ConvergenceSQ} shows the convergence of the policy evaluation algorithms under soft-quantile risk across different risk parameters and discount factors, under the setting of Section~\ref{sec:ConvergenceExperiementsd10}. The first two rows correspond to $d=N$, and the last two to $d<N$; shaded regions indicate standard errors. All methods converge to the true value function when $d=N$. UBSR-TD and UBSR-Newton converge to the same objective when $d< N$. UBSR-TD(1) does not minimize the distance to the projected value function for the case of $d=5$, $\gamma=0.05$, $\mu=0.8$, and diverges for the case of $d=5$, $\gamma=0.9$, and $\mu=0.8$.

\begin{figure}[htbp]
    \centering
    \includegraphics[width=0.95\linewidth]{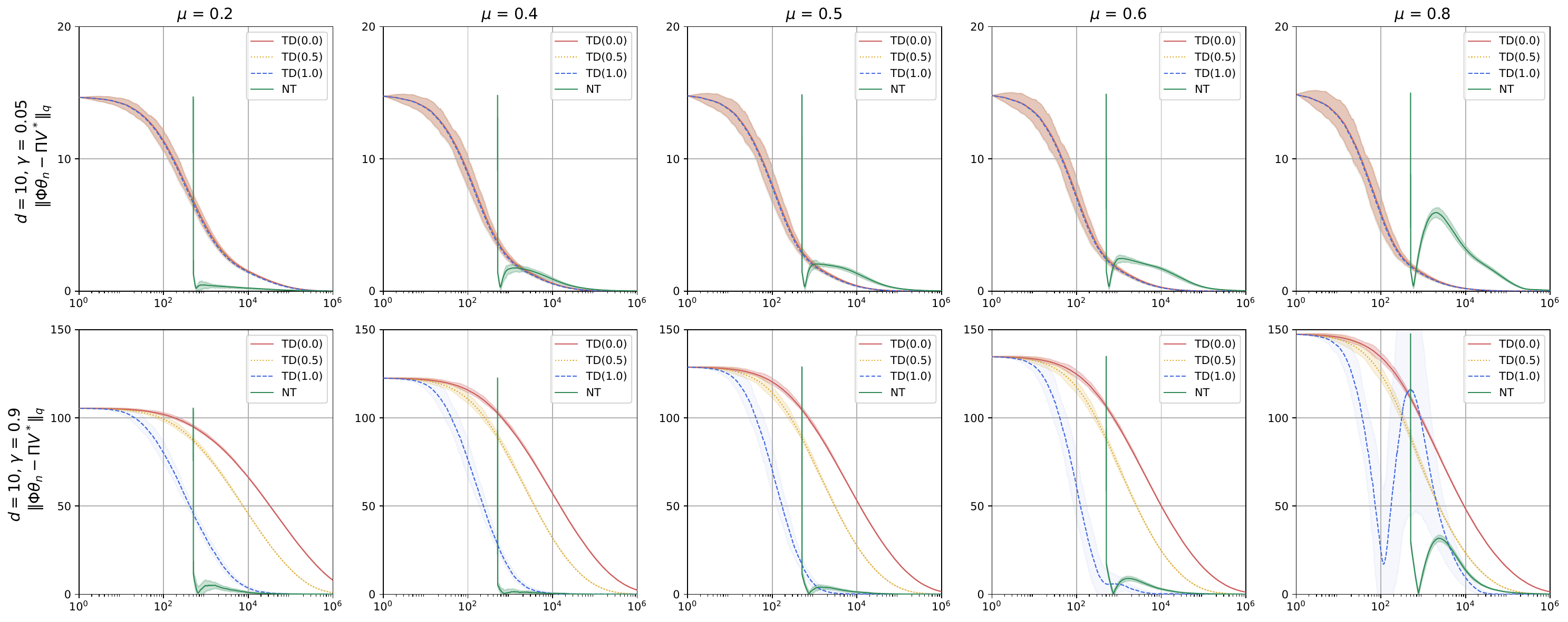}

    \includegraphics[width=0.95\linewidth]{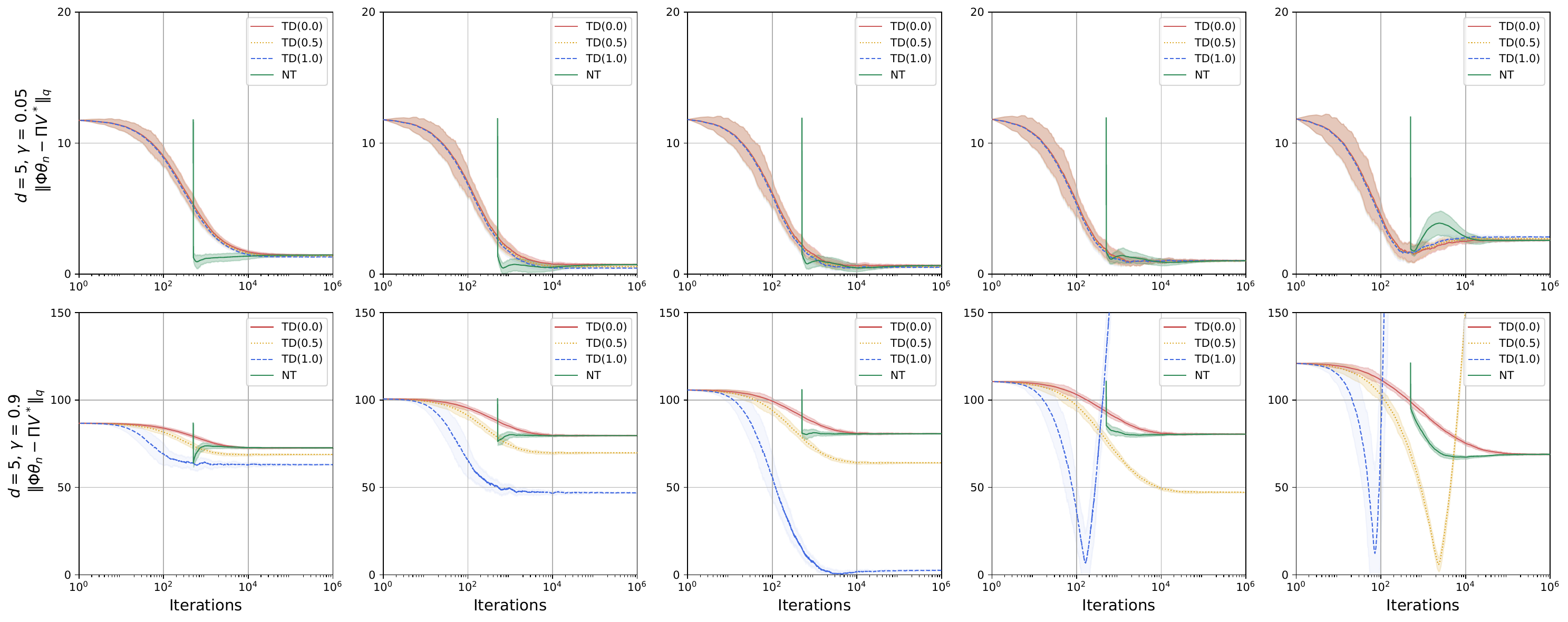}
    \caption{Convergence results for soft-quantile risk}
    \label{figure:ConvergenceSQ}
\end{figure}

\subsection{Runtime Comparison for Different ADP Policy Evaluation Algorithms}\label{sec:ADPPolicyEvaluation}

\begin{table}[htbp]
  \centering
  \caption{Compuational times of different policy evaluation algorithms}
  \resizebox{\textwidth}{!}{
    \begin{tabular}{c|cccccccccccccccccc}
    \toprule
    $m$     & \multicolumn{6}{c}{3}                         & \multicolumn{6}{c}{5}                         & \multicolumn{6}{c}{8} \\
    \midrule
    $\tau$   & \multicolumn{2}{c}{0.4} & \multicolumn{2}{c}{0.5} & \multicolumn{2}{c}{0.6} & \multicolumn{2}{c}{0.4} & \multicolumn{2}{c}{0.5} & \multicolumn{2}{c}{0.6} & \multicolumn{2}{c}{0.4} & \multicolumn{2}{c}{0.5} & \multicolumn{2}{c}{0.6} \\
    \midrule
    VI    & \multicolumn{2}{c}{5.74s} & \multicolumn{2}{c}{5.40s} & \multicolumn{2}{c}{7.48s} & \multicolumn{2}{c}{376.25s} & \multicolumn{2}{c}{396.71s} & \multicolumn{2}{c}{673.51s} & \multicolumn{2}{c}{$>$14h} & \multicolumn{2}{c}{$>$17h} & \multicolumn{2}{c}{$>$35h} \\
    TD    & 171.10s & (7.99\%) & 160.71s & (6.70\%) & 171.76s & (8.00\%) & 171.32s & (0.87\%) & 161.45s & (0.70\%) & 173.18s & (0.37\%) & 178.25s & (0.76\%) & 165.91s & (1.20\%) & 175.66s & (1.27\%) \\
    Newton & 198.12s & (2.07\%) & 183.31s & (0.45\%) & 198.98s & (2.74\%) & 201.36s & (2.48\%) & 186.43s & (0.62\%) & 202.82s & (7.70\%) & 216.02s & (3.12\%) & 197.30s & (0.40\%) & 212.77s & (7.44\%) \\
    \bottomrule
    \end{tabular}%
    }
  \label{table:PEComparison}%
\end{table}%

We compare the runtime of policy evaluation for the static policy in Section \ref{sec:PerishableInventoryProblem}. Table \ref{table:PEComparison} reports the average runtime (CPU time) of UBSR-TD and UBSR-Newton for one policy evaluation step (60,000 online sampling iterations) based on 20 independent runs, and compares it with exact policy evaluation by value iteration using a tolerance of $0.1$. Experiments are conducted using $m=3,5,8$, cost parameters $(c_1,c_2,c_3,c_4)=(10,1,20,5)$, and $\tau=0.4,0.5,0.6$. The values in parentheses denote the optimality gap between the averaged function approximation and the value computed by value iteration at the empty-inventory state.

The results show that the computational cost of value iteration grows rapidly with $m$, increasing from only a few seconds when $m=3$ to more than 10 hours when $m=8$. By contrast, UBSR-TD and UBSR-Newton require approximately 200 seconds across all settings while maintaining relative errors below $10\%$. Their runtimes increase moderately with $m$ as the number of features grows from 7 to 37. UBSR-Newton is consistently more computationally expensive than UBSR-TD because it additionally requires the processing of the information matrix. Overall, these results demonstrate that function approximation effectively alleviates the curse of dimensionality in the perishable inventory management problem.

\subsection{Further Details of ADP Policy Comparisons}\label{sec:FurtherDetailsADP}
Table~\ref{table:ADPPolicyEvaluationm3} reports the exact evaluations of policies obtained by value and policy iteration for the perishable inventory management problem under different risk and cost settings in Section \ref{sec:PerishableInventoryProblem} when $m=3$. Static, myopic, and risk-neutral (RN) policies are included as benchmarks, with optimality gaps shown in parentheses.

It can be seen that, the static policy performs poorly, with optimality gaps reaching $512.31\%$, highlighting the inadequacy of classical $(s,S)$ policies for perishable inventory. The myopic policy performs much better, with gaps below $3\%$ in most cases, but can still reach $30.49\%$ in the extreme case. The risk-neutral policy generally outperforms these benchmarks but can exceed a $6\%$ gap under strong risk sensitivity. In contrast, the policy iteration methods consistently achieve average gaps below $0.2\%$, demonstrating their effectiveness.

\begin{table}[htbp]
  \centering
  \caption{Results under varying policies, risk parameters, and cost settings for $m=3$}
  \resizebox{\textwidth}{!}{
    \begin{tabular}{c|c|cccccccccccccc}
    \toprule
    $\tau$   & Cost & \multicolumn{2}{c}{(10,1,20,5)} & \multicolumn{2}{c}{(20,1,20,5)} & \multicolumn{2}{c}{(50,1,20,5)} & \multicolumn{2}{c}{(80,1,20,5)} & \multicolumn{2}{c}{(10,1,20,20)} & \multicolumn{2}{c}{(10,1,20,50)} & \multicolumn{2}{c}{(10,1,20,80)} \\
    \midrule
    \multirow{6}[2]{*}{0.1} & \textbf{Optimal} & \multicolumn{2}{c}{\textbf{34.49 }} & \multicolumn{2}{c}{\textbf{54.02 }} & \multicolumn{2}{c}{\textbf{101.45 }} & \multicolumn{2}{c}{\textbf{144.53 }} & \multicolumn{2}{c}{\textbf{37.08 }} & \multicolumn{2}{c}{\textbf{40.47 }} & \multicolumn{2}{c}{\textbf{42.96 }} \\
          & Static & 36.52  & (5.91\%) & 59.77  & (10.65\%) & 122.80  & (21.04\%) & 181.45  & (25.54\%) & 47.08  & (26.98\%) & 68.03  & (68.09\%) & 88.68  & (106.44\%) \\
          & Myopic & 34.66  & (0.49\%) & 55.22  & (2.22\%) & 105.17  & (3.67\%) & 188.59  & (30.49\%) & 37.22  & (0.38\%) & 40.61  & (0.33\%) & 43.01  & (0.13\%) \\
          & RN    & 35.91  & (4.12\%) & 57.45  & (6.36\%) & 105.85  & (4.34\%) & 147.65  & (2.16\%) & 37.98  & (2.44\%) & 40.88  & (1.00\%) & 43.22  & (0.60\%) \\
          & TD    & 34.49  & (0.02\%) & 54.09  & (0.14\%) & 101.46  & (0.01\%) & 144.54  & (0.01\%) & 37.09  & (0.02\%) & 40.48  & (0.02\%) & 42.97  & (0.03\%) \\
          & Newton & 34.49  & (0.02\%) & 54.03  & (0.02\%) & 101.46  & (0.01\%) & 144.54  & (0.01\%) & 37.09  & (0.02\%) & 40.48  & (0.02\%) & 42.97  & (0.03\%) \\
    \midrule
    \multirow{6}[2]{*}{0.4} & \textbf{Optimal} & \multicolumn{2}{c}{\textbf{49.61 }} & \multicolumn{2}{c}{\textbf{72.43 }} & \multicolumn{2}{c}{\textbf{132.50 }} & \multicolumn{2}{c}{\textbf{184.28 }} & \multicolumn{2}{c}{\textbf{63.11 }} & \multicolumn{2}{c}{\textbf{79.09 }} & \multicolumn{2}{c}{\textbf{89.88 }} \\
          & Static & 66.26  & (33.56\%) & 113.77  & (57.07\%) & 204.79  & (54.56\%) & 253.68  & (37.66\%) & 137.07  & (117.20\%) & 278.31  & (251.91\%) & 419.40  & (366.61\%) \\
          & Myopic & 49.67  & (0.12\%) & 72.87  & (0.61\%) & 135.73  & (2.44\%) & 188.17  & (2.11\%) & 63.13  & (0.03\%) & 79.10  & (0.02\%) & 89.90  & (0.02\%) \\
          & RN    & 49.66  & (0.09\%) & 72.62  & (0.26\%) & 132.87  & (0.28\%) & 184.60  & (0.17\%) & 63.14  & (0.05\%) & 79.11  & (0.03\%) & 89.89  & (0.01\%) \\
          & TD    & 49.63  & (0.03\%) & 72.46  & (0.04\%) & 132.52  & (0.02\%) & 184.30  & (0.01\%) & 63.13  & (0.03\%) & 79.10  & (0.02\%) & 89.89  & (0.01\%) \\
          & Newton & 49.62  & (0.02\%) & 72.44  & (0.01\%) & 132.51  & (0.01\%) & 184.29  & (0.00\%) & 63.12  & (0.01\%) & 79.10  & (0.02\%) & 89.89  & (0.01\%) \\
    \midrule
    \multirow{5}[2]{*}{0.5} & \textbf{Optimal} & \multicolumn{2}{c}{\textbf{54.83 }} & \multicolumn{2}{c}{\textbf{78.15 }} & \multicolumn{2}{c}{\textbf{140.54 }} & \multicolumn{2}{c}{\textbf{194.79 }} & \multicolumn{2}{c}{\textbf{73.40 }} & \multicolumn{2}{c}{\textbf{94.37 }} & \multicolumn{2}{c}{\textbf{108.68 }} \\
          & Static & 72.94  & (33.03\%) & 122.58  & (56.85\%) & 216.48  & (54.04\%) & 267.53  & (37.34\%) & 158.48  & (115.90\%) & 329.56  & (249.20\%) & 500.63  & (360.66\%) \\
          & Myopic & 54.85  & (0.02\%) & 78.45  & (0.39\%) & 142.67  & (1.52\%) & 199.91  & (2.63\%) & 73.41  & (0.01\%) & 94.40  & (0.02\%) & 108.71  & (0.03\%) \\
          & TD    & 54.84  & (0.01\%) & 78.18  & (0.04\%) & 140.55  & (0.01\%) & 194.80  & (0.01\%) & 73.42  & (0.02\%) & 94.38  & (0.01\%) & 108.70  & (0.02\%) \\
          & Newton & 54.84  & (0.01\%) & 78.18  & (0.04\%) & 140.55  & (0.01\%) & 194.80  & (0.00\%) & 73.41  & (0.01\%) & 94.38  & (0.01\%) & 108.69  & (0.01\%) \\
    \midrule
    \multirow{6}[2]{*}{0.6} & \textbf{Optimal} & \multicolumn{2}{c}{\textbf{60.74 }} & \multicolumn{2}{c}{\textbf{84.54 }} & \multicolumn{2}{c}{\textbf{149.08 }} & \multicolumn{2}{c}{\textbf{205.56 }} & \multicolumn{2}{c}{\textbf{85.33 }} & \multicolumn{2}{c}{\textbf{112.22 }} & \multicolumn{2}{c}{\textbf{130.72 }} \\
          & Static & 89.50  & (47.34\%) & 144.49  & (70.92\%) & 230.27  & (54.47\%) & 285.22  & (38.75\%) & 215.41  & (152.46\%) & 467.55  & (316.62\%) & 719.79  & (450.65\%) \\
          & Myopic & 60.76  & (0.03\%) & 84.69  & (0.18\%) & 151.97  & (1.94\%) & 210.15  & (2.23\%) & 85.34  & (0.02\%) & 112.26  & (0.03\%) & 131.64  & (0.71\%) \\
          & RN    & 60.89  & (0.25\%) & 84.65  & (0.13\%) & 149.14  & (0.04\%) & 205.85  & (0.14\%) & 85.34  & (0.02\%) & 112.23  & (0.01\%) & 130.81  & (0.08\%) \\
          & TD    & 60.76  & (0.02\%) & 84.56  & (0.03\%) & 149.13  & (0.04\%) & 205.57  & (0.00\%) & 85.34  & (0.02\%) & 112.24  & (0.01\%) & 130.75  & (0.03\%) \\
          & Newton & 60.75  & (0.01\%) & 84.56  & (0.03\%) & 149.12  & (0.03\%) & 205.58  & (0.01\%) & 85.34  & (0.02\%) & 112.23  & (0.01\%) & 130.75  & (0.02\%) \\
    \midrule
    \multirow{6}[2]{*}{0.9} & \textbf{Optimal} & \multicolumn{2}{c}{\textbf{93.54 }} & \multicolumn{2}{c}{\textbf{118.31 }} & \multicolumn{2}{c}{\textbf{189.77 }} & \multicolumn{2}{c}{\textbf{256.15 }} & \multicolumn{2}{c}{\textbf{154.58 }} & \multicolumn{2}{c}{\textbf{213.98 }} & \multicolumn{2}{c}{\textbf{247.63 }} \\
          & Static & 148.14  & (58.37\%) & 218.22  & (84.45\%) & 284.89  & (50.13\%) & 350.61  & (36.88\%) & 421.61  & (172.75\%) & 968.93  & (352.81\%) & 1516.28  & (512.31\%) \\
          & Myopic & 93.54  & (0.01\%) & 118.34  & (0.03\%) & 190.34  & (0.30\%) & 258.55  & (0.94\%) & 156.07  & (0.96\%) & 214.13  & (0.07\%) & 247.80  & (0.07\%) \\
          & RN    & 100.69  & (7.65\%) & 126.39  & (6.83\%) & 196.78  & (3.70\%) & 264.99  & (3.45\%) & 156.19  & (1.04\%) & 214.63  & (0.30\%) & 263.37  & (6.36\%) \\
          & TD    & 93.56  & (0.02\%) & 118.34  & (0.03\%) & 189.84  & (0.04\%) & 256.25  & (0.04\%) & 154.60  & (0.01\%) & 214.01  & (0.01\%) & 247.69  & (0.02\%) \\
          & Newton & 93.55  & (0.02\%) & 118.34  & (0.03\%) & 189.83  & (0.03\%) & 256.17  & (0.01\%) & 154.59  & (0.01\%) & 213.99  & (0.00\%) & 247.66  & (0.01\%) \\
    \bottomrule
    \end{tabular}%
    }
  \label{table:ADPPolicyEvaluationm3}%
\end{table}%

Table \ref{table:ADPPolicyEvaluationm5full} presents the exact performance of the policies generated by value iteration and policy iteration under the various risk and cost specifications considered in Section \ref{sec:PerishableInventoryProblem} for $m=5$. Compared to the case of $m=3$, the static policy now performs substantially better than the myopic policy in most cases. In particular, it dominates the myopic policy when $\tau=0.4$ and $0.6$, and in several instances even slightly outperforms the policy obtained through policy iteration. As noted in \cite{Abouee2026ijoc}, this behavior can be attributed to the fact that, as $m$ increases, the perishable inventory model becomes closer to its non-perishable counterpart, with the limiting case $m\to\infty$ coinciding with the non-perishable inventory problem. Accordingly, for larger values of $m$, the optimal policy is expected to more closely resemble the static policy.

\begin{table}[htbp]
  \centering
  \caption{Results under varying policies, risk parameters, and cost settings for $m=5$}
  \resizebox{\textwidth}{!}{
    \begin{tabular}{c|c|cccccccccccccc}
    \toprule
    \multicolumn{1}{l|}{$\tau$} & Method & \multicolumn{2}{c}{(10,1,20,5)} & \multicolumn{2}{c}{(20,1,20,5)} & \multicolumn{2}{c}{(50,1,20,5)} & \multicolumn{2}{c}{(80,1,20,5)} & \multicolumn{2}{c}{(10,1,20,20)} & \multicolumn{2}{c}{(10,1,20,50)} & \multicolumn{2}{c}{(10,1,20,80)} \\
    \midrule
    \multirow{6}[1]{*}{0.1} & \textbf{Optimal} & \multicolumn{2}{c}{\textbf{31.23 }} & \multicolumn{2}{c}{\textbf{46.41 }} & \multicolumn{2}{c}{\textbf{84.81 }} & \multicolumn{2}{c}{\textbf{120.40 }} & \multicolumn{2}{c}{\textbf{31.26 }} & \multicolumn{2}{c}{\textbf{31.32 }} & \multicolumn{2}{c}{\textbf{31.38 }} \\
          & Static & 31.30  & (0.24\%) & 46.47  & (0.14\%) & 84.87  & (0.08\%) & 120.50  & (0.09\%) & 31.33  & (0.24\%) & 31.40  & (0.24\%) & 31.46  & (0.24\%) \\
          & Myopic & 31.86  & (2.02\%) & 48.40  & (4.28\%) & 92.69  & (9.29\%) & 188.46  & (56.53\%) & 31.88  & (1.99\%) & 31.92  & (1.90\%) & 31.95  & (1.81\%) \\
          & RN    & 33.18  & (6.27\%) & 48.52  & (4.55\%) & 86.41  & (1.89\%) & 121.66  & (1.05\%) & 33.23  & (6.31\%) & 33.15  & (5.85\%) & 33.17  & (5.70\%) \\
          & TD    & 31.30  & (0.24\%) & 46.47  & (0.14\%) & 84.94  & (0.16\%) & 120.70  & (0.25\%) & 31.33  & (0.24\%) & 31.40  & (0.24\%) & 31.46  & (0.24\%) \\
          & Newton & 31.30  & (0.24\%) & 46.47  & (0.14\%) & 84.87  & (0.08\%) & 120.41  & (0.01\%) & 31.33  & (0.24\%) & 31.40  & (0.25\%) & 31.46  & (0.24\%) \\
    \midrule
    \multirow{6}[0]{*}{0.4} & \textbf{Optimal} & \multicolumn{2}{c}{\textbf{38.44 }} & \multicolumn{2}{c}{\textbf{54.38 }} & \multicolumn{2}{c}{\textbf{95.02 }} & \multicolumn{2}{c}{\textbf{132.51 }} & \multicolumn{2}{c}{\textbf{38.75 }} & \multicolumn{2}{c}{\textbf{39.21 }} & \multicolumn{2}{c}{\textbf{39.64 }} \\
          & Static & 38.53  & (0.24\%) & 54.54  & (0.28\%) & 95.10  & (0.08\%) & 132.60  & (0.07\%) & 38.86  & (0.28\%) & 39.48  & (0.67\%) & 40.07  & (1.08\%) \\
          & Myopic & 40.03  & (4.15\%) & 60.18  & (10.66\%) & 113.22  & (19.15\%) & 162.65  & (22.75\%) & 40.18  & (3.68\%) & 40.44  & (3.12\%) & 40.69  & (2.63\%) \\
          & RN    & 38.60  & (0.43\%) & 54.59  & (0.39\%) & 95.10  & (0.08\%) & 132.60  & (0.07\%) & 38.94  & (0.50\%) & 39.44  & (0.58\%) & 39.88  & (0.60\%) \\
          & TD    & 38.53  & (0.24\%) & 54.46  & (0.14\%) & 95.10  & (0.08\%) & 132.60  & (0.07\%) & 38.85  & (0.26\%) & 39.42  & (0.53\%) & 39.76  & (0.29\%) \\
          & Newton & 38.53  & (0.24\%) & 54.46  & (0.14\%) & 95.10  & (0.08\%) & 132.60  & (0.07\%) & 38.85  & (0.25\%) & 39.31  & (0.25\%) & 39.74  & (0.25\%) \\
    \midrule
    \multirow{5}[0]{*}{0.5} & \textbf{Optimal} & \multicolumn{2}{c}{\textbf{40.23 }} & \multicolumn{2}{c}{\textbf{56.46 }} & \multicolumn{2}{c}{\textbf{97.67 }} & \multicolumn{2}{c}{\textbf{136.13 }} & \multicolumn{2}{c}{\textbf{40.79 }} & \multicolumn{2}{c}{\textbf{41.78 }} & \multicolumn{2}{c}{\textbf{42.55 }} \\
          & Static & 40.33  & (0.25\%) & 56.55  & (0.15\%) & 97.75  & (0.08\%) & 136.14  & (0.01\%) & 40.91  & (0.30\%) & 42.14  & (0.85\%) & 43.30  & (1.76\%) \\
          & Myopic & 41.23  & (2.47\%) & 61.50  & (8.92\%) & 115.07  & (17.82\%) & 165.49  & (21.57\%) & 41.54  & (1.83\%) & 42.21  & (1.03\%) & 42.82  & (0.63\%) \\
          & TD    & 40.43  & (0.49\%) & 56.55  & (0.15\%) & 97.75  & (0.08\%) & 136.14  & (0.01\%) & 40.89  & (0.26\%) & 41.90  & (0.29\%) & 42.64  & (0.21\%) \\
          & Newton & 40.43  & (0.49\%) & 56.55  & (0.15\%) & 97.75  & (0.08\%) & 136.14  & (0.01\%) & 40.89  & (0.26\%) & 41.96  & (0.43\%) & 42.63  & (0.17\%) \\
    \midrule
    \multirow{6}[1]{*}{0.6} & \textbf{Optimal} & \multicolumn{2}{c}{\textbf{42.38 }} & \multicolumn{2}{c}{\textbf{58.83 }} & \multicolumn{2}{c}{\textbf{100.60 }} & \multicolumn{2}{c}{\textbf{140.21 }} & \multicolumn{2}{c}{\textbf{43.37 }} & \multicolumn{2}{c}{\textbf{44.91 }} & \multicolumn{2}{c}{\textbf{46.23 }} \\
          & Static & 42.43  & (0.13\%) & 59.10  & (0.47\%) & 100.69  & (0.08\%) & 140.22  & (0.01\%) & 44.02  & (1.49\%) & 47.20  & (5.11\%) & 50.46  & (9.15\%) \\
          & Myopic & 44.01  & (3.84\%) & 64.59  & (9.79\%) & 121.33  & (20.60\%) & 173.77  & (23.93\%) & 44.46  & (2.50\%) & 45.45  & (1.21\%) & 46.51  & (0.60\%) \\
          & RN    & 42.55  & (0.39\%) & 58.98  & (0.27\%) & 100.93  & (0.32\%) & 140.37  & (0.11\%) & 43.49  & (0.27\%) & 45.00  & (0.21\%) & 46.33  & (0.22\%) \\
          & TD    & 42.40  & (0.05\%) & 58.93  & (0.18\%) & 100.69  & (0.08\%) & 140.23  & (0.01\%) & 43.45  & (0.18\%) & 45.04  & (0.28\%) & 46.49  & (0.55\%) \\
          & Newton & 42.40  & (0.05\%) & 58.91  & (0.15\%) & 100.69  & (0.08\%) & 140.22  & (0.01\%) & 43.45  & (0.19\%) & 45.17  & (0.58\%) & 46.33  & (0.21\%) \\
    \midrule
    \multirow{6}[2]{*}{0.9} & \textbf{Optimal} & \multicolumn{2}{c}{\textbf{52.98 }} & \multicolumn{2}{c}{\textbf{71.42 }} & \multicolumn{2}{c}{\textbf{119.55 }} & \multicolumn{2}{c}{\textbf{165.96 }} & \multicolumn{2}{c}{\textbf{61.58 }} & \multicolumn{2}{c}{\textbf{74.20 }} & \multicolumn{2}{c}{\textbf{84.09 }} \\
          & Static & 55.50  & (4.75\%) & 74.72  & (4.61\%) & 119.66  & (0.09\%) & 166.05  & (0.05\%) & 74.31  & (20.67\%) & 113.64  & (53.15\%) & 153.24  & (82.23\%) \\
          & Myopic & 54.19  & (2.27\%) & 75.45  & (5.64\%) & 139.19  & (16.42\%) & 200.82  & (21.00\%) & 62.75  & (1.89\%) & 82.17  & (10.73\%) & 101.66  & (20.90\%) \\
          & RN    & 66.04  & (24.65\%) & 82.62  & (15.68\%) & 126.23  & (5.59\%) & 170.86  & (2.95\%) & 67.94  & (10.32\%) & 76.07  & (2.52\%) & 87.66  & (4.24\%) \\
          & TD    & 54.13  & (2.16\%) & 71.77  & (0.49\%) & 119.65  & (0.08\%) & 166.04  & (0.05\%) & 62.66  & (1.76\%) & 75.33  & (1.52\%) & 84.56  & (0.56\%) \\
          & Newton & 53.79  & (1.54\%) & 71.60  & (0.25\%) & 119.65  & (0.08\%) & 166.04  & (0.05\%) & 61.68  & (0.16\%) & 74.32  & (0.16\%) & 84.14  & (0.06\%) \\
    \bottomrule
    \end{tabular}%
    }
  \label{table:ADPPolicyEvaluationm5full}%
\end{table}%

\end{document}